\documentclass[11pt]{article}

\usepackage[margin=1in]{geometry}
\usepackage{amsmath,amssymb,amsthm,mathtools}
\usepackage{enumitem}
\usepackage{algorithm,mathtools}
\usepackage{algpseudocode}
\usepackage{authblk}
\usepackage[numbers,sort&compress]{natbib}
\usepackage[colorlinks=true,linkcolor=blue,citecolor=blue,urlcolor=blue]{hyperref}

\algrenewcommand\algorithmicrequire{\textbf{Input:}}
\algrenewcommand\algorithmicensure{\textbf{Output:}}
\newcommand{\kb}[1]{{\color{red}\bf[KB: #1]}}

\newtheorem{theorem}{Theorem}[section]
\newtheorem{proposition}[theorem]{Proposition}
\newtheorem{corollary}[theorem]{Corollary}
\newtheorem{lemma}[theorem]{Lemma}
\newtheorem{definition}[theorem]{Definition}
\newtheorem{assumption}[theorem]{Assumption}
\newtheorem{remark}[theorem]{Remark}

\newcommand{\R}{\mathbb R}
\newcommand{\E}{\mathbb E}
\newcommand{\Prob}{\mathbb P}
\newcommand{\1}{\mathbf 1}
\newcommand{\eps}{\varepsilon}
\newcommand{\KL}{D_{\mathrm{KL}}}
\newcommand{\sgn}{\operatorname{sgn}}
\newcommand{\MCErr}{\operatorname{MCErr}}
\newcommand{\PredErr}{\operatorname{PredErr}}
\newcommand{\CumRes}{\operatorname{CumRes}}
\newcommand{\Err}{\operatorname{Err}}
\newcommand{\CVaR}{\operatorname{CVaR}}
\newcommand{\Var}{\operatorname{Var}}
\newcommand{\Unif}{\operatorname{Unif}}
\newcommand{\SC}{\operatorname{SC}}

\title{Sample Complexity of Multicalibration for Multilevel Properties}
\author[1]{Jiuyao Lu\thanks{This work was done while interning at Amazon.}}
\author[2,3]{Krishnakumar Balasubramanian}
\author[2]{Aleksandr Podkopaev}
\author[2]{Shiva Prasad Kasiviswanathan}

\affil[1]{Department of Statistics and Data Science,
The Wharton School, University of Pennsylvania}
\affil[2]{Amazon}
\affil[3]{Department of Statistics, University of California, Davis}
\date{}

\begin{document}
\maketitle

\begin{abstract}
Calibration requires a predictor to be unbiased after conditioning on its own predictions.  Multicalibration asks for this guarantee simultaneously across a collection of groups.  Many prediction tasks ask for several related features of the same conditional outcome distribution: variance is defined relative to the mean, skewness relative to both mean and variance, and conditional value at risk relative to a quantile.  We study multicalibration for a sequence of $k$ properties in which each property is identifiable once the preceding properties are fixed.  This framework includes Bayes pairs but does not require the properties to arise from a single loss.

For every fixed $k\ge2$, we establish matching upper and lower sample-complexity bounds up to logarithmic factors under regularity conditions.  Even with only polylogarithmically many binary groups, achieving multicalibration error $\varepsilon$ requires $\widetilde{\Omega}(\varepsilon^{-(k+2)})$ samples.  Conversely, for any finite group family $\mathcal G$, we give a randomized learner using $O(\varepsilon^{-(k+2)}+\varepsilon^{-2}\log|\mathcal G|)$ samples.  Thus the sample complexity is $\widetilde{\Theta}(\varepsilon^{-(k+2)})$ for polynomial-size group families.  We instantiate the theory for three canonical examples.
\end{abstract}

\tableofcontents

\section{Introduction}

Calibration asks whether predictions have their intended statistical meaning.  For a predictor of the conditional mean, the basic requirement is that among the individuals who receive prediction \(p\), the average outcome is also \(p\)~\citep{dawid1982well}.  This guarantee can still hide systematic errors: a predictor may underestimate outcomes on one structured subpopulation and compensate by overestimating them elsewhere.  Given a family \(\mathcal G\) of groups, multicalibration requires the predictor to remain calibrated after restricting attention to any one of them~\citep{hebert2018multicalibration}. Multicalibration is useful when the same predictions inform many downstream decisions.  It supports predictors that perform well for many losses, as in omniprediction~\citep{gopalan2021omnipredictors,gopalan2023loss,okoroafor2025near}, and it has found applications in complexity theory and information aggregation~\citep{casacuberta2024complexity,dwork2025supersimulators,collina2025tractable,collina2026collaborative}.
%This raises a basic question about sample complexity: 

Prior work has developed multicalibration for conditional means, moments, and quantiles~\citep{jung2021moment,gupta2021online,jung2023batch,deng2023happymap,noarov2023scope}.  In many applications, however, several related features of the conditional outcome law must be predicted together, and some are defined relative to others.
Quantiles describe coverage and tail events, dispersion measures describe uncertainty, and measures of tail risk summarize rare outcomes with potentially large costs.  For example, a three-level example arises in hospital capacity planning, where a model may need to predict a patient's expected length of stay, the variability around that expectation, and the degree of right-tail asymmetry caused by occasional very long stays.  The model can output a prediction vector \(p=(m,\sigma^2,\gamma)\), representing the conditional mean, variance, and skewness of length of stay.  These quantities have a natural sequential structure.  At the first level, the mean prediction \(m\) separates patient groups whose stays are centered at different values.  At the second level, after the mean has been fixed, the variance prediction \(\sigma^2\) describes variability around that mean, so that differences in expected length of stay are not mistaken for uncertainty.  At the third level, after the mean and variance have been fixed, the skewness prediction \(\gamma\) describes asymmetry relative to that center and scale.  The third coordinate is important because two patient groups may have the same expected stay and the same overall variability but very different tail behavior: one may have roughly symmetric fluctuations around its mean, while the other may usually be discharged near the expected date but occasionally require a much longer hospitalization.  Calibrating skewness without also fixing the mean and variance could pool groups with different centers or scales, causing those lower-order differences to appear spuriously as tail asymmetry.  Three-level multicalibration therefore ensures that each prediction describes a distinct feature of the conditional outcome distribution: its location, its dispersion around that location, and the shape of its tail after location and dispersion have been accounted for.

Formally, let \(\mathcal{M}\) be a class of probability laws on the outcome space \(\mathcal{Y}\), and let \(\Gamma_j:\mathcal{M}\to\mathbb{R}\) denote the \(j\)-th distributional property of interest.  For each context \(x\), write \(\mu_x=\mathcal{L}(Y\mid X=x)\) for the conditional law of the outcome.  Calibrating the properties \(\Gamma_1,\Gamma_2,\Gamma_3\) separately need not yield a jointly coherent predictor.  Indeed, consider the oracle scalar predictor \(f_j(x)=\Gamma_j(\mu_x)\).  Conditioning on the event \(f_j(X)=a\) pools the conditional laws \(\mu_x\) into the mixture
\[
\bar{\mu}_a
=
\mathcal{L}(Y\mid f_j(X)=a)
=
\int \mu_x\,d\mathbb{P}(x\mid f_j(X)=a).
\]
Although every component of this mixture satisfies \(\Gamma_j(\mu_x)=a\), calibration additionally requires \(\Gamma_j(\bar{\mu}_a)=a\).  This implication holds only when the level set $\Gamma_j^{-1}(a)
=
\{\mu\in\mathcal{M}:\Gamma_j(\mu)=a\}$ is closed under mixtures.  Thus, if \(\Gamma_j\) lacks convex level sets, even an oracle predictor can fail to be calibrated after conditioning on its own output~\citep{noarov2023scope}.  In a multilevel setting, the failure may arise because distributions that agree on a higher-level property \(\Gamma_j\), but differ in lower-level properties \(\Gamma_1,\ldots,\Gamma_{j-1}\), are pooled together, and these lower-level differences change the value of \(\Gamma_j\) under mixing.

Variance provides a concrete example of this failure.  Suppose
\[
\Prob[(X,Y)=(1,1)]
=
\Prob[(X,Y)=(2,2)]
=
\frac12,
\]
and consider the groups
\[
g_1(x)=\1\{x=1\},
\qquad
g_2(x)=\1\{x=2\},
\qquad
g_{\mathrm{all}}(x)\equiv1.
\]
The oracle conditional-variance predictor outputs zero at both contexts because the outcome is deterministic conditional on \(X\).  It is calibrated within \(g_1\) and \(g_2\), but not within \(g_{\mathrm{all}}\).  Its single prediction cell pools the two outcomes, whose variance is \(1/4\).  In fact, \(g_1\) and \(g_2\) force any variance-only predictor to output zero at both contexts, whereas calibration within \(g_{\mathrm{all}}\) would require their common prediction to be \(1/4\).  Thus no scalar variance predictor can be multicalibrated with respect to all three groups.  Predicting the mean together with the variance separates the two contexts through their mean coordinates and removes this conflict.

This motivates predicting the ordered vector $f(x)=\bigl(f_1(x),f_2(x),f_3(x)\bigr)$ and calibrating its coordinates jointly.  For simplicity, suppose that there are residual functions \(R_1,R_2,R_3\) such that \(R_1\) identifies \(\Gamma_1\), \(R_2\) identifies \(\Gamma_2\) once the value of \(\Gamma_1\) is fixed, and \(R_3\) identifies \(\Gamma_3\) once the values of both \(\Gamma_1\) and \(\Gamma_2\) are fixed.  That is,
\[
\mathbb{E}_{\mu}\!\left[R_1(p_1,Y)\right]=0
\quad\Longleftrightarrow\quad
p_1=\Gamma_1(\mu),
\]
and, conditional on \(p_1=\Gamma_1(\mu)\),
\[
\mathbb{E}_{\mu}\!\left[R_2(p_1,p_2,Y)\right]=0
\quad\Longleftrightarrow\quad
p_2=\Gamma_2(\mu),
\]
while, conditional on \(p_1=\Gamma_1(\mu)\) and \(p_2=\Gamma_2(\mu)\),
\[
\mathbb{E}_{\mu}\!\left[R_3(p_1,p_2,p_3,Y)\right]=0
\quad\Longleftrightarrow\quad
p_3=\Gamma_3(\mu).
\]
Three-level multicalibration then requires these residuals to have mean zero on every relevant subgroup and within each prediction cell indexed by \(f=(f_1,f_2,f_3)\).
The earlier coordinates therefore prevent contexts with different predictions at earlier levels from being pooled when calibrating a later coordinate, allowing a property that is not calibratable in isolation to become calibratable relative to the preceding predictions.

%We study the sample complexity when this dependence continues for an arbitrary fixed number of levels.  We consider a \(k\)-level property \(\Gamma=(\Gamma_1,\ldots,\Gamma_k)\), with \(k\ge2\), for which level \(j\) is identifiable after the values of levels \(1,\ldots,j-1\) have been fixed.  The residual at level \(j\) may therefore depend on the entire preceding prediction.  Calibration is imposed jointly: we condition on the entire prediction vector and require the residuals at all levels to be calibrated on every group.  This differs from collecting \(k\) unrelated scalar predictors, whose calibration conditions may use incompatible buckets and do not account for how an error at one level changes the residuals at later levels.

This motivates us to study the sample complexity of multilevel multicalibration under an ECE-type metric defined from these residuals, where ECE denotes expected calibration error~\citep{naeini2015obtaining,collina2026sample}.  Let \((\Gamma_1,\ldots,\Gamma_k)\) be an ordered collection of distributional properties such that, for each level \(j\in[k]\), the property \(\Gamma_j\) is identifiable conditional on the preceding properties \((\Gamma_1,\ldots,\Gamma_{j-1})\).  Given \(n\) independent samples and a target accuracy \(\eps\), we ask when a learner can produce a predictor \(f=(f_1,\ldots,f_k)\) with multicalibration error at most \(\eps\) over every relevant subgroup.
Our goal is to characterize how the required sample size \(n\) scales with \(\eps\) and the number of groups, for each fixed \(k\).
In particular, we ask whether exploiting the conditional structure permits sample complexity comparable to scalar or vector-valued mean multicalibration, or whether the dependence among levels introduces an unavoidable additional statistical cost.

%\citet{noarov2023scope} developed joint multicalibration for pairs in which the second property is elicitable conditional on the first; Bayes pairs form an important special case.  We ask for the sample complexity of this conditional structure under the ECE metric used by \citet{collina2026sample} for scalar and high-dimensional mean prediction. \kb{Don't understand this sentence. Aren't we asking for more?}

\subsection{Main Results}

We prove upper and lower bounds on the sample complexity of multicalibration that match up to logarithmic factors for sequentially conditionally identifiable multilevel properties.  For finitely supported predictors, our multicalibration error is the maximum over groups of the \(\ell_1\) expected calibration error computed by bucketing according to the full prediction vector.  (See Section~\ref{sec:setup} for the definitions of sequential conditional identifiability and multicalibration error for general predictors.)  The lower and upper bounds rely on different regularity assumptions, all of which are satisfied by our canonical examples.

\paragraph{Lower Bound.}
For every sufficiently small \(\eps\), we construct a finite context space, polylogarithmically many binary groups, and a collection of data distributions such that any learner achieving multicalibration error at most \(\eps\) with probability at least \(2/3\) must have sample size
\[
n
=
\Omega\!\left(
\frac{\eps^{-(k+2)}}{\log^{k+2}(1/\eps)}
\right)
=
\widetilde\Omega\!\left(\eps^{-(k+2)}\right).
\]
The result holds against randomized learners.  Since the constructed group family has polylogarithmic size, it lies within every group-size budget of the form \(\eps^{-\kappa}\) with fixed \(\kappa>0\), once \(\eps\) is sufficiently small.

\paragraph{Upper Bound.}
We give an online forecaster and convert its sequence of prediction rules into a randomized predictor.  For a finite family \(\mathcal G\) of groups, the resulting learner uses
\[
n
=
O\!\left(
\eps^{-(k+2)}
+
\eps^{-2}\log|\mathcal G|
\right).
\]
Thus, whenever \(|\mathcal G|\) is polynomial in \(1/\eps\), the upper and lower bounds have the same exponent \(k+2\).  We also give conditions under which the learner can be implemented in polynomial time.

\paragraph{Canonical Properties.}
We verify the conditions of both bounds in three cases: mean and mean absolute deviation; mean, variance, and skewness; and quantile and conditional value at risk.  The two-level examples have sample complexity \(\widetilde\Theta(\eps^{-4})\), while the three-level example has sample complexity \(\widetilde\Theta(\eps^{-5})\).

The results of \citet{collina2026sample} are most closely related to ours.  They establish the exponent \(3\) for scalar mean multicalibration and \(d+2\) for \(d\)-dimensional vector means, and also treat regular scalar elicitable properties.  When \(\Gamma(\mu)=\E_\mu[\psi(Y)]\) for a fixed vector-valued function \(\psi\), vector mean multicalibration applies directly by treating \(\psi(Y)\) as the outcome.  More generally, a multilevel property need not admit such a representation, and a residual at a later level may depend on the earlier predictions.  Our bounds allow this structure.

\subsection{Proof Overview}

\paragraph{Lower Bound.}
The lower bound proceeds by considering the task of determining which distribution generated the samples.  The proof has four steps.  First, we construct a large finite collection of candidate distributions.  The true distribution is an unknown member of this collection, and the samples are drawn from it.  Second, we show that a predictor with small multicalibration error must have small prediction error.  Third, we use such a predictor to determine which member generated the samples.  Finally, we show that determining the true distribution requires many samples.

Fix a resolution parameter \(Q\).  We take the context space to be the discrete grid \(\{0,\ldots,Q-1\}^k\), which contains \(Q^k\) contexts.  We associate each context with an unperturbed vector in the \(k\)-dimensional property space; these vectors form a Cartesian grid.  We assign one hidden sign in \(\{-1,+1\}\) to each context.  According to this sign assignment, we shift the final coordinate of the corresponding property vector by an amount of order \(1/Q\), while leaving the first \(k-1\) coordinates unchanged.  We select \(\exp(\Omega(Q^k))\) sign assignments such that any two of them disagree on a constant fraction of the contexts.  Our regularity assumptions (Assumption~\ref{ass:witness}(i)) ensure that every resulting vector is the property value of some outcome distribution.  Taking the context to be uniform and using the corresponding outcome distribution conditionally on each context produces one candidate data distribution for every sign assignment.  The resulting statistical task is to identify which candidate distribution generated the data, or equivalently, its sign assignment.

The second step converts multicalibration error into prediction error.  Consider a given level, and suppose first that the predictions at all preceding levels are correct.  Under our regularity assumptions (Assumption~\ref{ass:witness}(ii)), the expected residual at the current level has the same sign as the difference between the prediction at that level and its true value, and its magnitude is at least proportional to the absolute value of this difference.  Incorrect predictions at preceding levels can change the current residual, but the change is controlled by the corresponding prediction errors (see Assumption~\ref{ass:witness}(iv)).  Thus, the signed residual controls the current prediction error up to the prediction errors at preceding levels.  The required sign depends on both the prediction and the context, whereas a group can depend only on the context.  For a fixed prediction, however, this sign is a threshold function of one context coordinate.  We construct a common collection of only polylogarithmically many binary groups whose linear combinations approximate all the threshold functions needed in the proof (see Lemma~\ref{lem:walsh}).  Multiplying the residual by such an approximation produces, up to the approximation error, a linear combination of group-weighted residuals.  Each of these residual terms is controlled by the multicalibration error (see Lemmas~\ref{lem:signed-weights} and~\ref{lem:variable-walsh-tv}).  Since the sum of the absolute values of the coefficients is \(O(\log Q)\), the signed residual, and hence the current prediction error, is controlled by \(O(\log Q)\) times the multicalibration error, together with the prediction errors at preceding levels.

Because a residual at level \(j\) may depend on all preceding predictions, we apply this argument successively.  After controlling the errors in the first \(j-1\) coordinates, the Lipschitz dependence of the level-\(j\) residual on preceding predictions allows us to control the error in coordinate \(j\).  This gives bounds for the first \(k-1\) coordinates.  For the final coordinate, we additionally separate predictions according to whether they lie inside or outside disjoint neighborhoods of the unperturbed property vectors.  Threshold functions control the predictions outside these neighborhoods, while the group containing all contexts controls those inside.  Altogether, multicalibration error at most \(\eps\) forces the expected \(\ell_1\) distance between the prediction and the true property vector to be \(O(\eps\log Q)\), where the expectation is over a uniformly random context and the predictor's randomization (see Proposition~\ref{prop:lower-anticoarsening}).

The third step uses this prediction guarantee to recover the sign assignment.  Because any two candidate sign assignments disagree on a constant fraction of the contexts, their associated property vectors are separated by order \(1/Q\) on average.  Given a possibly randomized predictor, we take its mean output at every context and select the sign assignment whose associated property vectors are closest to these mean predictions.  If the average prediction error is smaller than a sufficiently small constant multiple of \(1/Q\), the true sign assignment is the unique closest one (see Lemmas~\ref{lem:lower-separation} and~\ref{lem:lower-decoding}).  Consequently, any learner that produces a predictor with multicalibration error at most \(\eps\) also provides a way to determine the true distribution whenever
\[
\eps\log Q \lesssim \frac1Q.
\]

Finally, our regularity assumptions (Assumption~\ref{ass:witness}(v)) ensure that the Kullback--Leibler divergence between nearby outcome distributions is at most a constant times the squared distance between their property vectors.  Since the candidates differ by only \(O(1/Q)\) at each context, the pairwise Kullback--Leibler divergence between their one-sample data distributions is \(O(1/Q^2)\), and for \(n\) independent samples it is \(O(n/Q^2)\) (see Lemma~\ref{lem:lower-kl}).  On the other hand, the logarithm of the number of candidates is \(\Omega(Q^k)\).  Using Fano's inequality, we show in Proposition~\ref{prop:lower-resolution} that identifying the true candidate with constant success probability requires
\[
\frac{n}{Q^2}\gtrsim Q^k,
\qquad\text{and hence}\qquad
n=\Omega(Q^{k+2}).
\]
Choosing \(Q\) at the largest scale allowed by \(\eps\log Q\lesssim 1/Q\), namely
\[
Q\asymp\frac{1}{\eps\log(1/\eps)},
\]
gives
\[
n
=
\Omega\left(
\frac{\eps^{-(k+2)}}{\log^{k+2}(1/\eps)}
\right).
\]
The constructed family contains only polylogarithmically many groups, so it lies within every polynomial group-size budget for sufficiently small \(\eps\).

\paragraph{Upper Bound.}
Our upper-bound construction proceeds through an online forecasting problem.  Fix a finite grid \(\mathcal P_Q\) of \(k\)-dimensional prediction vectors, where \(Q\) controls the resolution and \(|\mathcal P_Q|=O(Q^k)\).  At round \(t\), the forecaster uses the history from the preceding rounds to construct a randomized prediction rule
\[
\pi_t:\mathcal X\to\Delta(\mathcal P_Q),
\]
where \(\Delta(\mathcal P_Q)\) denotes the set of distributions over the grid.  The current context \(X_t\) is then revealed, and the forecaster predicts the distribution \(\pi_t(\cdot\mid X_t)\).  After observing the outcome \(Y_t\), it evaluates the residuals associated with this distribution.
We first control the empirical analogue of multicalibration error accumulated over these online rounds.

Given a dataset consisting of \(T\) i.i.d. context--outcome pairs, we run this online forecasting procedure for \(T\) rounds, using the \(t\)-th pair on round \(t\).  The context \(X_t\) is revealed before the prediction and the outcome \(Y_t\) afterward.  This run produces \(T\) prediction rules \(\pi_1,\ldots,\pi_T\).  We return their average:
\[
\Pi(x)=\frac1T\sum_{t=1}^T\pi_t(\cdot\mid x).
\]
An online-to-batch argument transfers the empirical guarantee for the online transcript to a population multicalibration guarantee for \(\Pi\).

The empirical objective contains many calibration requirements.  For every group \(g\), point \(p\in\mathcal P_Q\), and level \(j\), it records the cumulative level-\(j\) residual associated with \(g\), with the residual on round \(t\) weighted by the probability \(\pi_t(p\mid X_t)\) assigned to \(p\).  For each group, the empirical error sums the absolute values of these cumulative residuals over all pairs \((p,j)\).  We can express this sum as a maximum over sign arrays \(s=(s_{p,j})\), with one sign assigned to each pair \((p,j)\).  After also taking the maximum over groups, the empirical error becomes the maximum of a collection of signed residual objectives indexed by pairs \((g,s)\).  Controlling all of these objectives is therefore equivalent to controlling the empirical multicalibration error.

The forecaster assigns a weight to each objective at the beginning of every round.  These weights depend on the signed residuals accumulated on previous rounds.  If a calibration violation has grown large for some group and some pattern of signs across prediction cells and levels, the corresponding objective receives greater weight, so the next prediction places more emphasis on reducing that violation.  Objectives whose cumulative violations remain small receive less relative emphasis.  Once the weights have been fixed, every candidate distribution on \(\mathcal P_Q\) has a weighted expected residual under each possible conditional outcome law.  The forecaster chooses among these distributions in a minimax manner by minimizing the largest such weighted expected residual over all admissible outcome laws.

Formally, we use the framework of online learning with expert advice.  We treat each objective as an expert and its signed residual on each round as the expert's gain.  The exponential weights algorithm assigns weights to the experts according to their cumulative gains, producing the objective weights described above.  Its regret guarantee bounds the largest cumulative gain of any expert by the sum of the expected gains under these weights, plus a regret term (Lemma~\ref{lem:upper-exponential-weights}).  Because the empirical multicalibration error is exactly the largest expert gain divided by \(T\), it remains to control these expected gains.  Conditional on the past and the current context, the expected gain on each round is bounded by the value of the minimax problem solved on that round.

We bound this value by reversing the order of play using Sion's minimax theorem (Lemma~\ref{lem:upper-minimax-value}).  In the reversed game, an admissible outcome law \(\mu\) is fixed before the prediction is selected.  For high-dimensional mean prediction, it suffices to choose a grid point close to the revealed outcome vector~\citep{noarov2025high,collina2026sample}.  In our setting, the residual need not be the difference between a prediction and an outcome, and a residual at a later level may depend on the entire preceding prediction.  Instead, we assume that, for every \(\mu\), there is a single grid point at which all expected residuals are simultaneously \(O(1/Q)\) (Assumption~\ref{ass:upper-analytic}(iii)).  Choosing this point in the reversed game makes every weighted objective \(O(1/Q)\).  The minimax theorem then guarantees that, in the original order of play, there is a distribution on \(\mathcal P_Q\) whose worst-case weighted expected residual is equally small, even though the forecaster does not know \(\mu\).

Combining this bound with the regret guarantee of exponential weights, Theorem~\ref{thm:upper-online} gives
\[
O\left(
\frac1Q+
\sqrt{\frac{Q^k+\log|\mathcal G|}{T}}
\right)
\]
expected empirical multicalibration error, apart from the chosen accuracy for solving the minimax problem.
% Although the number of sign arrays is exponential in \(|\mathcal P_Q|\), only its logarithm enters the regret bound: an expert is indexed by one group and \(k|\mathcal P_Q|\) signs, so the logarithm of the number of experts is \(O(\log|\mathcal G|+Q^k)\).

Finally, we convert the online guarantee into a batch guarantee.  A martingale concentration argument shows that the population multicalibration error of the averaged predictor \(\Pi\) exceeds the empirical multicalibration error of the online transcript by at most a term of the same order (Lemma~\ref{lem:upper-online-to-batch}).  Taking \(Q=\Theta(1/\eps)\) and
\[
T=
O\left(
\eps^{-(k+2)}
+
\eps^{-2}\log|\mathcal G|
\right)
\]
gives the upper bound on sample complexity in Corollary~\ref{cor:upper-sample}.

Although the expert family is exponentially large, its weights can be computed without enumerating all experts.  The minimax step reduces to linear optimization over \(\mathcal M\), as specified in Assumption~\ref{ass:upper-oracle}.  For each of our canonical examples, we implement this oracle in polynomial time (Propositions~\ref{prop:upper-mad}, \ref{prop:upper-mvs}, and~\ref{prop:upper-qc}), yielding the polynomial-time learner in Theorem~\ref{thm:upper-polytime}.

\subsection{Related Work}

\paragraph{Multicalibration and Sample Complexity.}
\citet{hebert2018multicalibration} introduced multicalibration.  Subsequent work developed online algorithms and efficient implementations, and connected multicalibration to multiobjective optimization and omniprediction~\citep{lee2022online,haghtalab2023unifying,garg2024oracle,ghugeimproved}.  Other formulations measure calibration differently.  \citet{gopalan2022low} consider a multicalibration condition that tests residuals after restricting attention to observations whose predictions fall in an interval.  \citet{globus2023boosting} study multicalibration under a weighted \(L_2\) metric and give a boosting algorithm for controlling it.  When \(|\mathcal G|\) is polynomial in \(1/\eps\), the guarantees of \citet{gopalan2022low} and \citet{globus2023boosting} translate into sample complexities of \(\widetilde O(\eps^{-8})\) and \(\widetilde O(\eps^{-10})\) for mean ECE, respectively~\citep{collina2026sample}.

\citet{gopalan2023swap} introduce swap multicalibration, a stronger condition in which the group function used to test residual bias may vary across prediction values.  \citet{luo2025swap} improve the online rates and sample-complexity upper bounds for swap multicalibration of conditional means when residual bias is tested using bounded linear functions of the context.

The weaker notion of calibrated multiaccuracy asks a predictor to be globally calibrated and multiaccurate with respect to a class of groups, meaning that its residuals are uncorrelated with every group function in the class~\citep{casacuberta2025global}.  \citet{gibbs-tibshirani-omnipredict} establish an \(\Omega(\eps^{-5/2})\) sample-complexity lower bound for learning deterministic predictors of conditional means under this requirement.

\citet{collina2026sample} study batch sample complexity for scalar and vector means, weighted error metrics, and regular scalar elicitable properties.  \citet{collina2026optimal} study the online setting, establish sharp lower bounds, and develop the subsampled Walsh family used in our group construction.

A related but distinct statistical question is whether empirical multicalibration error reliably estimates population multicalibration error throughout a chosen predictor class.  \citet{shabat2020uniform} establish bounds depending on the size or graph dimension of the class, while \citet{rosenberg2022exploration} connect this question to standard ERM generalization bounds.  This guarantee concerns estimation within the class, not the existence of a low-error predictor in it.

\paragraph{Calibration of Distributional Properties.}
\citet{jung2021moment} introduce mean-conditioned moment multicalibration, and \citet{gupta2021online} give an online algorithm for this notion.  Multivalid prediction intervals have also been studied in online and batch settings~\citep{gupta2021online,bastani2022practical,jung2023batch}.  \citet{noarov2023scope} connect property multicalibration to elicitation and identification.  They characterize which continuous scalar properties are sensible for calibration under mild conditions and give an algorithm for jointly multicalibrating two-level conditionally elicitable properties, including Bayes pairs.
\citet{hu2025efficient} give an oracle-efficient online algorithm for swap multicalibration of elicitable properties.  We study the sample complexity of joint calibration for an arbitrary fixed number of conditional levels, including sequences that are not generated by a single loss and its Bayes risk.

\paragraph{High-Dimensional Prediction.}
Recent work studies calibration and related unbiasedness guarantees for vector-valued predictions, including computational questions and online rates that depend on dimension~\citep{gopalan2024multiclass,peng2025high,fishelson2026high,noarov2025high}.  \citet{noarov2025high} give an online algorithm for predicting high-dimensional states such that, on every event in a specified family, the accumulated prediction error is small in every coordinate.  These events may depend on the past transcript, the current context, and the prediction.  In our setting, each coordinate predicts one property in an ordered sequence rather than one component of a state vector, and the residual for a later property may depend on the preceding predictions.

\paragraph{Organization.}
Section~\ref{sec:setup} defines sequential conditional identifiability and the multicalibration error used throughout.  Sections~\ref{sec:lower} and~\ref{sec:upper} prove the lower and upper bounds, respectively.  Section~\ref{sec:instantiations} verifies the assumptions in our three canonical examples.

\section{Basic Setup}\label{sec:setup}

\subsection{Sequential Conditional Identifiability}

Let \(\mathcal Y\) be a standard Borel outcome space and let \(\mathcal M\) be a class of probability measures on \(\mathcal Y\).  Fix an integer \(k\ge2\) and a compact rectangular prediction space
\[
\mathcal P
=
\mathcal P_1\times\cdots\times\mathcal P_k
\subset\R^k.
\]
Throughout, \(k\) is treated as fixed.  Constants suppressed by asymptotic
notation may depend on \(k\), as well as on the stated regularity parameters.
A prediction value is denoted \(p=(p_1,\ldots,p_k)\).  For \(j\in[k]\), write
\[
p_{<j}:=(p_1,\ldots,p_{j-1}),
\qquad
\mathcal P_{<j}:=\mathcal P_1\times\cdots\times\mathcal P_{j-1},
\]
with \(\mathcal P_{<1}\) interpreted as a one-point set.

\begin{definition}[Sequentially Conditionally Identifiable Property]
A \(k\)-level property is a map
\[
\Gamma=(\Gamma_1,\ldots,\Gamma_k):\mathcal M\to\mathcal P.
\]
For \(\mu\in\mathcal M\), write
\[
\Gamma_{<j}(\mu):=(\Gamma_1(\mu),\ldots,\Gamma_{j-1}(\mu)).
\]
The property \(\Gamma\) is sequentially conditionally identifiable if, for every level \(j\in[k]\), there is a jointly measurable residual function
\[
R_j:\mathcal P_{<j}\times\mathcal P_j\times\mathcal Y\to\R
\]
with the following properties.  For every \(\mu\in\mathcal M\) and every \((a,q)\in\mathcal P_{<j}\times\mathcal P_j\), the function \(R_j(a,q,\cdot)\) is \(\mu\)-integrable.  Moreover, for every \(\mu\in\mathcal M\) and every \(q\in\mathcal P_j\),
\[
\E_{Y\sim\mu}R_j(\Gamma_{<j}(\mu),q,Y)=0
\quad\Longleftrightarrow\quad
q=\Gamma_j(\mu).
\]
\end{definition}

Thus level \(j\) is identified after the previous true levels have been fixed.  The residual functions are also evaluated away from the true prefix, because predictions need not have correct earlier coordinates.  The residual vector used throughout the paper is
\[
R(p,y)
:=
\bigl(R_1(p_{<1},p_1,y),\ldots,R_k(p_{<k},p_k,y)\bigr).
\]
When the final argument is a distribution, expectation over that distribution is denoted by:
\[
R_j(a,q,\mu)
:=
\E_{Y\sim\mu}R_j(a,q,Y),
\qquad
R(p,\mu)
:=
\E_{Y\sim\mu}R(p,Y).
\]
Sequential conditional identifiability implies
\[
R(\Gamma(\mu),\mu)=0
\qquad
\forall \mu\in\mathcal M.
\]

The familiar Bayes pair setting is the special case \(k=2\)~\citep{noarov2023scope}.  If a scalar property \(\gamma\) has identification function \(r_\gamma\), and if \(L\) is a strictly consistent loss for \(\gamma\) (so \(\gamma(\mu)\) uniquely minimizes \(\E_\mu L(r,Y)\)), then its Bayes risk is
\[
\mathcal B(\mu):=\E_{Y\sim\mu}L(\gamma(\mu),Y),
\]
and the Bayes pair is \(\Gamma(\mu)=(\gamma(\mu),\mathcal B(\mu))\).  It is sequentially conditionally identifiable with residual functions
\[
R_1(\varnothing,p_1,y)=r_\gamma(p_1,y),
\qquad
R_2(p_1,p_2,y)=p_2-L(p_1,y).
\]
For \(\tau\in(0,1)\), the quantile and conditional-value-at-risk example in
Section~\ref{sec:quantile-cvar} uses the loss
\[
L_\tau(q,y)
:=
q+\frac{(y-q)_+}{1-\tau}
\]
which is strictly consistent for \(q_\tau\) on the distribution class
considered there and has Bayes risk \(\CVaR_\tau\).

\subsection{Multicalibration Error}

Let \(\mathcal X\) be a standard Borel context space and let \(\mathsf D\) be a distribution on \(\mathcal X\times\mathcal Y\).  Let \(\mathsf D_{Y\mid X=x}\) denote a regular conditional law of \(Y\) given \(X=x\).  We call \(\mathsf D\) \(\mathcal M\)-compatible if
\[
\mathsf D_{Y\mid X=x}\in\mathcal M
\qquad
\text{for \(\mathsf D_X\)-almost every \(x\)}.
\]
A randomized predictor assigns to each context \(x\in\mathcal X\) a distribution \(\Pi_x\in\Delta(\mathcal P)\) over prediction values.  This assignment is measurable in the sense that, for every Borel set \(A\subseteq\mathcal P\), the map \(x\mapsto\Pi_x(A)\) is measurable.  We write
\[
\Pi=(\Pi_x)_{x\in\mathcal X}.
\]

For a given \(\mathsf D\) and \(\Pi\), consider the absolute-integrability condition
\begin{equation}\label{eq:residual-integrability}
\E_{(X,Y)\sim\mathsf D}
\left[
\int_{\mathcal P}\|R(p,Y)\|_1\,\Pi_X(dp)
\right]
<\infty.
\end{equation}
When \eqref{eq:residual-integrability} holds, a measurable weight \(w:\mathcal X\to[-1,1]\) defines a finite vector signed measure on \(\mathcal P\) by
\[
\begin{aligned}
\nu^{\mathsf D,\Pi}_{w}(A)
&:=
\E_{(X,Y)\sim\mathsf D}
\left[
w(X)\int_A R(p,Y)\,\Pi_X(dp)
\right]\\
&=
\E_{\substack{(X,Y)\sim\mathsf D\\
P\sim\Pi_X}}
\left[
w(X)\mathbf 1\{P\in A\}
R(P,Y)
\right],
\qquad \text{for measurable \(A\subseteq\mathcal P\)}.
\end{aligned}
\]

The population expected calibration error for \(\Gamma\) (\(\Gamma\)-ECE) of \(\Pi\) with respect to \(w\) is the coordinatewise total variation
\[
\begin{aligned}
\Err^\Gamma_{\mathsf D}(\Pi;w)
&:=
\sum_{j=1}^k
|\nu^{\mathsf D,\Pi}_{w,j}|(\mathcal P)\\
&=
\sum_{j=1}^k
\E_{\substack{(X,Y)\sim\mathsf D\\
P\sim\Pi_X}}
\left[
\left|
\E\left[
w(X)R_j(P_{<j},P_j,Y)
\mid P
\right]
\right|
\right].
\end{aligned}
\]
If \(R\) is not absolutely integrable under \((\mathsf D,\Pi)\), we set \(\Err^\Gamma_{\mathsf D}(\Pi;w)=+\infty\).
A group function is a measurable map \(g:\mathcal X\to[0,1]\).  It is binary if it takes values in \(\{0,1\}\).  For a finite family \(\mathcal G\) of group functions, define
\[
\MCErr^\Gamma_{\mathsf D}(\Pi;\mathcal G)
:=
\max_{g\in\mathcal G}
\Err^\Gamma_{\mathsf D}(\Pi;g).
\]

\begin{remark}[Bucketing by the Full Prediction Vector]
When \(\Pi\) is supported on a finite set \(S(\Pi)\subset\mathcal P\), the total-variation definition becomes
\[
\Err^\Gamma_{\mathsf D}(\Pi;g)
=
\sum_{p\in S(\Pi)}
\left\|
\E\left[
g(X)\Pi_X(\{p\})
R(p,Y)
\right]
\right\|_1.
\]
Thus each bucket is indexed by \(p=(p_1,\ldots,p_k)\), rather than by one coordinate at a time.  \citet{noarov2023scope} use a different error metric for two-level joint multicalibration.  For each coordinate, they fix the other prediction coordinate and average squared calibration errors over the values of the coordinate being evaluated.  Our ECE metric follows the definition for scalar and high-dimensional mean prediction studied by \citet{collina2026sample}, summing the absolute residual mass over prediction values and levels.  This joint bucketing gives the \(k\)-dimensional exponent.
\end{remark}

\subsection{Minimax Sample Complexity}\label{sec:minimax-sample-complexity}

We use the minimax formulation for batch multicalibration of
\citet{collina2026sample}, with the distributions restricted to those that
are \(\mathcal M\)-compatible.  A learner \(\mathsf A=(\mathsf A_n)_{n\ge1}\)
receives a finite group family \(\mathcal G\), a sample
\(S\in(\mathcal X\times\mathcal Y)^n\), and an independent random seed
\(\zeta\), and returns a randomized predictor
\(\mathsf A_n(\mathcal G,S,\zeta)\) with prediction space \(\mathcal P\).

Let \(B:(0,1)\to\mathbb N\) be a group-size budget.  For a learner
\(\mathsf A\), let \(n_{\mathsf A}^\Gamma(\eps;B,\mathcal M)\) be the smallest
\(n\ge1\) such that, for every standard Borel context space \(\mathcal X\),
every \(\mathcal M\)-compatible distribution \(\mathsf D\) on
\(\mathcal X\times\mathcal Y\), and every finite group family
\(\mathcal G\) with \(1\le |\mathcal G|\le B(\eps)\),
\[
\Prob_{S\sim\mathsf D^n,\,\zeta}
\left[
\MCErr^\Gamma_{\mathsf D}
\bigl(\mathsf A_n(\mathcal G,S,\zeta);\mathcal G\bigr)
\le\eps
\right]
\ge\frac23.
\]
If no such \(n\) exists, set
\(n_{\mathsf A}^\Gamma(\eps;B,\mathcal M)=+\infty\).  The minimax sample
complexity is
\[
\SC_{\Gamma,\mathcal M}(\eps;B)
:=
\inf_{\mathsf A}n_{\mathsf A}^\Gamma(\eps;B,\mathcal M).
\]
For a fixed \(\kappa>0\), define
\[
B_\kappa(\eps):=\left\lceil\eps^{-\kappa}\right\rceil,
\qquad
\SC^{(\kappa)}_{\Gamma,\mathcal M}(\eps)
:=
\SC_{\Gamma,\mathcal M}(\eps;B_\kappa).
\]
Constants suppressed by asymptotic notation may depend on the fixed parameter \(\kappa\).
Thus the context space, distribution, and group family may all vary with
\(\eps\), subject to the compatibility and size requirements above.

\section{The Lower Bound}\label{sec:lower}

This section proves the lower bound by reducing multicalibration to the task of identifying which distribution in a finite collection generated the observed data.  We construct this collection so that a predictor with small multicalibration error is accurate enough to reveal the true member.  We then use an information-theoretic argument to show that this identification cannot be accomplished from too few samples.

The construction must therefore satisfy two requirements.  The same candidate distributions must yield sufficiently different property vectors at the contexts for an accurate predictor to distinguish them, yet remain statistically close enough that distinguishing them from the observed data requires many samples.  We also construct a small family of groups such that small multicalibration error with respect to these groups guarantees the required prediction accuracy.  The regularity assumptions below make these ingredients possible.

\subsection{Regularity Assumptions}

Sequential conditional identifiability determines the value at which the expected residual at each level vanishes, but it does not quantify how the residual changes when the prediction moves away from that value.  For the lower bound, we require this quantitative control for a family of outcome distributions indexed by property vectors in a fixed rectangle \(\mathcal R_0\).  Specifically, for every \(v\in\mathcal R_0\), the family contains a distribution \(\mu_v\) on \(\mathcal Y\) satisfying \(\Gamma(\mu_v)=v\) (Condition~(i) below).  We will later use these outcome distributions as conditional laws at different contexts to construct the finite collection of candidate data distributions.

The remaining conditions govern how the residuals and distributions vary within this family.  When the predictions at preceding levels are correct, the expected residual at the current level must have the same sign as the current prediction error, and its magnitude must be bounded above and below by constant multiples of that error (Conditions~(ii) and~(iii)).  Errors in preceding predictions may change a later residual, but only in proportion to those errors (Condition~(iv)).  Finally, distributions with nearby property vectors must have small Kullback--Leibler divergence, so that the candidate distributions we construct are difficult to distinguish (Condition~(v)).

Formally, let
\[
\mathcal R_0=I_1^0\times\cdots\times I_k^0
\subset
\operatorname{int}(\mathcal P)
\]
be a nondegenerate compact rectangle.  A point in \(\mathcal R_0\) is denoted
\[
v=(v_1,\ldots,v_k).
\]
Write
\[
I_j^0=[\ell_j,\ell_j+W_j],
\qquad
W_j>0,
\qquad j\in[k].
\]
For \(v\in\mathcal R_0\), write \(v_{<j}:=(v_1,\ldots,v_{j-1})\).  We interpret \(v_{<1}\) as the empty prefix.  Thus \(R_j(v_{<j},q,\mu_v)\) is the expected level-\(j\) residual evaluated with the true prefix.

\begin{assumption}[Regularity of the Witness Family]\label{ass:witness}
There exists a family of outcome distributions
\[
\{\mu_v:v\in\mathcal R_0\}\subset\mathcal M
\]
and constants \(c_{\mathrm{anti}},C_{\mathrm{prev}},C_{\mathrm{lip}},C_{\mathrm{KL}}\in(0,\infty)\) such that the following hold.

\begin{enumerate}[label=(\roman*)]
\item For every \(v\in\mathcal R_0\),
\[
\Gamma(\mu_v)=v.
\]

\item For every \(v\in\mathcal R_0\), every \(j\in[k]\), and every \(q\in\mathcal P_j\),
\[
\sgn(q-v_j)\left(R_j(v_{<j},q,\mu_v)-R_j(v_{<j},v_j,\mu_v)\right)
\ge
c_{\mathrm{anti}}|q-v_j|.
\]
Here and throughout this section, we use the convention \(\sgn(0)=0\).

\item For every \(v\in\mathcal R_0\), every \(j\in[k]\), and every \(q\in\mathcal P_j\),
\[
\left|R_j(v_{<j},q,\mu_v)-R_j(v_{<j},v_j,\mu_v)\right|
\le
C_{\mathrm{lip}}|q-v_j|.
\]

\item For every \(v\in\mathcal R_0\), every \(p\in\mathcal P\), and every \(j\in[k]\),
\[
\left|R_j(p_{<j},p_j,\mu_v)-R_j(v_{<j},p_j,\mu_v)\right|
\le
C_{\mathrm{prev}}\sum_{i<j}|p_i-v_i|.
\]
The sum over \(i<1\) is interpreted as zero.

\item For every \(v,v'\in\mathcal R_0\),
\[
\KL(\mu_v\,\|\,\mu_{v'})
\le
C_{\mathrm{KL}}\|v-v'\|_2^2.
\]
\end{enumerate}
\end{assumption}

\citet{collina2026sample} use counterparts of Conditions~(i), (ii), and~(v) in their lower bounds for scalar properties.  Our proof additionally uses~(iii) to bound the magnitude of a residual by the current prediction error and~(iv) to bound the effect of errors at earlier levels.

As a quick sanity check, for mean and mean absolute deviation, the expected level-\(j\) residual at the true prefix is exactly \(q-v_j\).  Thus, (ii) and~(iii) hold as equalities with constant \(1\), and~(iv) also holds with constant \(1\).  For quantile and CVaR, (ii) and~(iii) in the quantile coordinate follow directly from lower and upper bounds on the density, while in the CVaR coordinate they hold as equalities with constant \(1\).  Section~\ref{sec:instantiations} gives the complete verifications.

\subsection{Main Result}

The conditions above allow us to construct a finite collection of candidate data distributions for every sufficiently small \(\eps\).  Any learner that achieves multicalibration error at most \(\eps\) on every distribution in the collection must use \(\widetilde{\Omega}(\eps^{-(k+2)})\) samples.  The corresponding group family has only polylogarithmically many members and therefore satisfies \(|\mathcal G_\eps|\le\eps^{-\kappa}\) for every fixed \(\kappa>0\) once \(\eps\) is sufficiently small.  The following theorem states this result precisely.

\begin{theorem}[Lower Bound for Multicalibration of \(k\)-Level Properties]\label{thm:lower-bound}
Suppose Assumption~\ref{ass:witness} holds.  Fix any \(\kappa>0\).  Then, for every sufficiently small \(\eps>0\), there exist

\begin{enumerate}[label=(\roman*)]
\item a finite context space \(\mathcal X_\eps\);
\item a binary group family \(\mathcal G_\eps\) on \(\mathcal X_\eps\) satisfying
\[
|\mathcal G_\eps|
\le
\eps^{-\kappa};
\]
\item a finite collection \(\mathfrak D_\eps\) of \(\mathcal M\)-compatible distributions on \(\mathcal X_\eps\times\mathcal Y\);
\end{enumerate}
such that any possibly randomized learner which, given \(n\) i.i.d. samples from any \(\mathsf D\in\mathfrak D_\eps\), outputs a predictor \(\Pi\) satisfying
\[
\MCErr^\Gamma_{\mathsf D}(\Pi;\mathcal G_\eps)
\le
\eps
\]
with probability at least \(2/3\), must use
\[
n
=
\Omega\!\left(
\frac{\eps^{-(k+2)}}{\log^{k+2}(1/\eps)}
\right)
=
\widetilde{\Omega}\!\left(\eps^{-(k+2)}\right).
\]
The implicit constants, and the threshold for sufficiently small \(\eps\), depend only on \(k\), \(\kappa\), \(\mathcal R_0\), and the constants in Assumption~\ref{ass:witness}.
\end{theorem}

Equivalently, in the notation of Section~\ref{sec:minimax-sample-complexity},
\[
\SC^{(\kappa)}_{\Gamma,\mathcal M}(\eps)
=
\Omega\!\left(
\frac{\eps^{-(k+2)}}{\log^{k+2}(1/\eps)}
\right)
=
\widetilde{\Omega}\!\left(\eps^{-(k+2)}\right).
\]

We prove Theorem~\ref{thm:lower-bound} in the remainder of this section, starting with the construction of the candidate distributions.

\subsection{The Candidate Distributions}

Fix a power of two \(Q\ge2\).  We take the context space to be the discrete grid
\[
\mathcal X_Q=\{0,\dots,Q-1\}^k.
\]
We write a context as
\[
a=(a_1,\ldots,a_k)\in\mathcal X_Q.
\]

Recall that \(I_j^0=[\ell_j,\ell_j+W_j]\) is the \(j\)-th coordinate interval of \(\mathcal R_0\), with lower endpoint \(\ell_j\) and width \(W_j>0\).
For \(j\in[k]\) and \(u\in\{0,\dots,Q-1\}\), define
\[
b_j(u)
:=
\ell_j+\frac{W_j}{3}
+
\frac{u\,W_j}{3(Q-1)}.
\]
Thus \(b_j(0),\ldots,b_j(Q-1)\) are \(Q\) evenly spaced values in the \(j\)-th property coordinate.
For \(a\in\mathcal X_Q\), define the unperturbed property vector associated with context \(a\)
\[
b(a)
:=
(b_1(a_1),\ldots,b_k(a_k)).
\]
The vectors \(b(a)\), as \(a\) ranges over \(\mathcal X_Q\), form a Cartesian grid in the property space.  This grid lies in the central third
\[
\prod_{j=1}^k
\left[\ell_j+\frac{W_j}{3},\,\ell_j+\frac{2W_j}{3}\right]
\subset \mathcal R_0.
\]
The distance between consecutive values \(b_j(u)\) is \(W_j/[3(Q-1)]\).  Let \(d_Q\) be the smallest of these distances:
\[
d_Q
:=
\frac{\min_{j\in[k]}W_j}{3(Q-1)}.
\]
In particular,
\[
d_Q\asymp Q^{-1},
\]
with constants depending only on \(\mathcal R_0\).

As in the construction used by \citet{collina2026sample} for vector means, we assign one hidden sign to each context.  At each context, this sign determines the direction in which we perturb the final property coordinate.  Keeping the preceding coordinates at their unperturbed values allows their prediction errors to be controlled successively.  The direction in which the final coordinate is perturbed can be chosen separately at each of the \(Q^k\) contexts, yielding exponentially many candidate data distributions.

We choose the hidden perturbation to be small relative to \(d_Q\).  This keeps every perturbed property vector inside the witness rectangle and ensures that the property vectors assigned to different contexts remain well separated.
Choose a constant \(c_\eta>0\) satisfying
\[
c_\eta
<
\frac{1}{20},
\]
and set
\[
\eta:=c_\eta d_Q.
\]
Then, for every \(a\in\{0,\dots,Q-1\}^k\) and every \(\theta_a\in\{-1,+1\}\),
\[
b(a)+\eta\theta_a e_k\in\mathcal R_0.
\]

We need to choose exponentially many sign assignments while ensuring that any two disagree on a constant fraction of the contexts.  Let
\[
\Theta_Q\subseteq\{-1,+1\}^{\mathcal X_Q}
\]
satisfy
\[
\log|\Theta_Q|\ge c_{\mathrm{pack}}Q^k,
\qquad
\bigl|\{a\in\mathcal X_Q:\theta_a\neq\theta'_a\}\bigr|
\ge\rho_{\mathrm{pack}} Q^k
\quad
\forall \theta\neq\theta',
\]
for universal constants \(c_{\mathrm{pack}},\rho_{\mathrm{pack}}>0\).  Such a collection exists by the Gilbert packing bound~\citep{roth2006introduction,ta2017explicit}.

Each sign assignment specifies the true property vector at every context, and hence one candidate distribution.  For \(\theta\in\Theta_Q\), define the true property vector at context \(a\) by
\[
v_\theta(a)
:=
b(a)+\eta\theta_a e_k.
\]
Thus
\[
v_{\theta,j}(a)=b_j(a_j)\quad(j<k),
\qquad
v_{\theta,k}(a)=b_k(a_k)+\eta\theta_a.
\]
The distribution \(\mathsf D_\theta\) on \(\mathcal X_Q\times\mathcal Y\) is
\[
X\sim\Unif(\mathcal X_Q),
\qquad
Y\mid X=a\sim\mu_{v_\theta(a)}.
\]

We next construct a group family for which small multicalibration error forces a predictor's outputs to be close to \(v_\theta(a)\).  This prediction guarantee will allow us to identify \(\theta\).

\subsection{Approximation of Threshold Signs with Walsh Functions}

The key relation, formalized later in Lemma~\ref{lem:triangular-consequences}, has the schematic form
\[
|p_j-v_{\theta,j}(a)|
\lesssim
\sgn\!\bigl(p_j-v_{\theta,j}(a)\bigr)
R_j\!\left(p_{<j},p_j,\mu_{v_\theta(a)}\right)
+
\sum_{i<j}|p_i-v_{\theta,i}(a)|.
\]
Thus, once the errors in the preceding coordinates have been controlled, the error in coordinate \(j\) can be bounded through a residual multiplied by the sign of the current error.

Multicalibration, on the other hand, controls residuals multiplied by groups \(g(a)\).  We therefore seek to approximate the sign in the inequality above by a linear combination of a common family of functions of the context.  This will allow us to bound the prediction error by the multicalibration error multiplied by a factor.  The sum of the absolute values of the coefficients in the linear combination determines this factor, while the number of functions determines the size of the resulting group family.

For \(j<k\), we have \(v_{\theta,j}(a)=b_j(a_j)\), so for a fixed prediction \(p_j\), the sign
\[
a\mapsto\sgn\!\bigl(p_j-b_j(a_j)\bigr)
\]
is a threshold function of \(a_j\).  For the final coordinate, the two possible true values \(b_k(a_k)-\eta\) and \(b_k(a_k)+\eta\) lie in a small interval around \(b_k(a_k)\).  Outside a slightly larger interval, the sign of the error is the same for both values.  We use a two-sided threshold that agrees with this common sign outside the interval and is zero inside it, where predictions will be handled separately.

Because the group family must be fixed independently of the prediction and of \(\theta\), we need a common set of functions that can represent all these one-sided and two-sided threshold signs, with uniformly controlled coefficients.  The full Walsh basis gives an exact representation with coefficient bound \(O(\log Q)\), but it would require \(O(Q)\) groups.  By retaining a suitable common subset of the Walsh functions, we can uniformly approximate all the required threshold signs while preserving the \(O(\log Q)\) coefficient bound.  Only polylogarithmically many Walsh functions are needed.

Because \(Q\) is a power of two, identify each \(u,\ell\in\{0,\ldots,Q-1\}\) with its length-\(\log_2 Q\) binary expansion.  Define the Walsh function
\[
\psi_\ell(u)
:=
(-1)^{\langle \ell,u\rangle_{\mathbb F_2}}.
\]
Here
\[
\langle \ell,u\rangle_{\mathbb F_2}
:=
\sum_{i=1}^{\log_2 Q}\ell_i u_i
\pmod 2.
\]
The Walsh functions form an orthonormal basis for real-valued functions on \(\{0,\ldots,Q-1\}\).

For \(\sigma\in\{-1,+1\}\) and \(t\in\R\), define the one-sided function
\[
s_{\sigma,t}(u)
:=
\sigma\sgn(t-u).
\]
For \(\sigma\in\{-1,+1\}\) and \(t_-\le t_+\), define the two-sided function
\[
s_{\sigma,t_-,t_+}(u)
:=
\frac{\sigma}{2}\bigl(\sgn(t_- -u)+\sgn(t_+ -u)\bigr).
\]
Let \(\mathcal T_Q\) be the class of all these functions.
A one-sided function changes sign at one threshold.  A two-sided function is zero between two thresholds and has opposite signs beyond them, with the boundary values determined by \(\sgn(0)=0\).  A two-sided function tests whether the final prediction lies outside a prescribed interval.

\citet{collina2026optimal} show that one common polylogarithmic subset of the Walsh basis approximates every one-sided discrete threshold sign.  Their coefficients satisfy the uniform \(O(\log Q)\) bound in Lemma~\ref{lem:walsh}.

\begin{lemma}[Walsh Approximation of Threshold Signs]\label{lem:walsh}
There is a universal constant \(C_{\mathrm W}\) such that the following holds.  Let \(Q\ge2\) be a power of two.  For every fixed accuracy \(\alpha\in(0,1)\), there exists a set
\[
\mathcal S_\alpha\subseteq\{1,\dots,Q-1\}
\]
of indices of Walsh functions satisfying
\[
|\mathcal S_\alpha|
\le
C_{\mathrm W}\alpha^{-2}\log^3(Q+1)
\]
and coefficient functions
\[
\beta_0:\mathcal T_Q\to\R,
\qquad
\beta_\ell:\mathcal T_Q\to\R
\quad(\ell\in\mathcal S_\alpha),
\]
such that, for every \(s\in\mathcal T_Q\),
\[
\widehat s(u)
=
\beta_0(s)+\sum_{\ell\in\mathcal S_\alpha}\beta_\ell(s)\psi_\ell(u)
\]
is a uniform approximation:
\[
\|s-\widehat s\|_\infty\le\alpha
\]
and the coefficients satisfy
\[
\sup_{s\in\mathcal T_Q}|\beta_0(s)|
+
\sum_{\ell\in\mathcal S_\alpha}
\sup_{s\in\mathcal T_Q}|\beta_\ell(s)|
\le
C_{\mathrm W}\log(Q+1).
\]
\end{lemma}

\begin{proof}
For each \(r\in\{0,\ldots,Q\}\), define
\[
f_r:\{0,\ldots,Q-1\}\to\{-1,+1\}
\]
by
\[
f_r(u)
:=
\begin{cases}
+1, & u<r,\\
-1, & u\ge r.
\end{cases}
\]
\citet{collina2026optimal} construct a common set
\[
\mathcal S_\alpha\subseteq\{1,\ldots,Q-1\},
\qquad
|\mathcal S_\alpha|
\le
C\alpha^{-2}\log^3(Q+1),
\]
together with coefficients \(a_0(r)\) and \(c_\ell(r)\) for all \(r\in\{0,\ldots,Q\}\).
The corresponding approximations
\[
\widetilde f_r(u)
:=
a_0(r)
+
\sum_{\ell\in\mathcal S_\alpha}
c_\ell(r)\psi_\ell(u)
\]
satisfy
\[
\max_{r\in\{0,\ldots,Q\}}
\|f_r-\widetilde f_r\|_\infty
\le
\alpha
\]
and
\[
\max_{r\in\{0,\ldots,Q\}}|a_0(r)|
+
\sum_{\ell\in\mathcal S_\alpha}
\max_{r\in\{0,\ldots,Q\}}|c_\ell(r)|
\le
C\log(Q+1)
\]
for a universal constant \(C\).

We use these coefficients to define \(\beta_0\) and \(\beta_\ell\) on \(\mathcal T_Q\).  For each function \(s\) that has a one-sided representation, fix one such representation \(s=s_{\sigma,t}\).
If \(t\notin\{0,\ldots,Q-1\}\), the restriction of \(u\mapsto\sgn(t-u)\) to \(\{0,\ldots,Q-1\}\) equals a unique \(f_r\).  Define
\[
\beta_0(s)
:=
\sigma a_0(r),
\qquad
\beta_\ell(s)
:=
\sigma c_\ell(r)
\quad(\ell\in\mathcal S_\alpha).
\]
If \(t\in\{0,\ldots,Q-1\}\), then
\[
\sgn(t-u)
=
\frac{f_t(u)+f_{t+1}(u)}{2}.
\]
In this case, define
\[
\beta_0(s)
:=
\frac{\sigma}{2}\bigl(a_0(t)+a_0(t+1)\bigr)
\]
and
\[
\beta_\ell(s)
:=
\frac{\sigma}{2}\bigl(c_\ell(t)+c_\ell(t+1)\bigr)
\quad(\ell\in\mathcal S_\alpha).
\]

For each function \(s\in\mathcal T_Q\) not already covered, fix a two-sided representation \(s=s_{\sigma,t_-,t_+}\).  By definition,
\[
s_{\sigma,t_-,t_+}(u)
=
\frac{s_{\sigma,t_-}(u)+s_{\sigma,t_+}(u)}2.
\]
Using the coefficients already defined for these two one-sided functions, set
\[
\beta_0(s)
:=
\frac{\beta_0(s_{\sigma,t_-})+\beta_0(s_{\sigma,t_+})}2
\]
and
\[
\beta_\ell(s)
:=
\frac{\beta_\ell(s_{\sigma,t_-})+\beta_\ell(s_{\sigma,t_+})}2
\quad(\ell\in\mathcal S_\alpha).
\]

In each case, \(\widehat s\) is obtained by applying the same multiplication and averaging operations to the corresponding approximations \(\widetilde f_r\).  The triangle inequality therefore gives
\[
\|s-\widehat s\|_\infty
\le
\alpha.
\]
Every coefficient \(\beta_0(s)\) is a signed average of coefficients \(a_0(r)\), and every \(\beta_\ell(s)\) is the corresponding signed average of coefficients \(c_\ell(r)\).  Consequently,
\[
\sup_{s\in\mathcal T_Q}|\beta_0(s)|
\le
\max_{r\in\{0,\ldots,Q\}}|a_0(r)|
\]
and, for every \(\ell\in\mathcal S_\alpha\),
\[
\sup_{s\in\mathcal T_Q}|\beta_\ell(s)|
\le
\max_{r\in\{0,\ldots,Q\}}|c_\ell(r)|.
\]
The required coefficient bound now follows from the corresponding bound on \(a_0(r)\) and \(c_\ell(r)\), after enlarging the universal constant if necessary.
\end{proof}

\subsection{The Group Family}

We now use the Walsh approximation above to construct a family of binary groups.

Fix a sufficiently small constant \(\alpha\in(0,1)\), and let \(\mathcal S_\alpha\) be the set supplied by Lemma~\ref{lem:walsh}.  For each coordinate \(j\in[k]\) and each \(\ell\in\mathcal S_\alpha\), define the signed Walsh weight
\[
w_{j,\ell}(a)
:=
\psi_\ell(a_j).
\]
Convert it into two binary groups
\[
g^+_{j,\ell}(a)
:=
\frac{1+w_{j,\ell}(a)}2,
\qquad
g^-_{j,\ell}(a)
:=
\frac{1-w_{j,\ell}(a)}2.
\]
Let
\[
g_{\mathrm{all}}\equiv1,
\]
and define
\[
\mathcal G_Q
:=
\{g_{\mathrm{all}}\}
\cup
\{g^+_{j,\ell},g^-_{j,\ell}:\ j\in[k],\ \ell\in\mathcal S_\alpha\}.
\]
Then \(\mathcal G_Q\) is binary-valued and satisfies
\[
|\mathcal G_Q|
\le
C_{\alpha,k}\log^3(Q+1).
\]

The analysis uses the signed Walsh weights \(w_{j,\ell}\), although the group family itself is binary-valued.  The following elementary lemma passes between these two forms.

\begin{lemma}[Signed Weights from Binary Groups]\label{lem:signed-weights}
For every \(j\in[k]\) and \(\ell\in\mathcal S_\alpha\), every distribution \(\mathsf D\) on \(\mathcal X_Q\times\mathcal Y\), and every predictor \(\Pi\),
\[
\Err^\Gamma_{\mathsf D}(\Pi;w_{j,\ell})
\le
2\,\MCErr^\Gamma_{\mathsf D}(\Pi;\mathcal G_Q).
\]
\end{lemma}

\begin{proof}
Since \(w_{j,\ell}=g^+_{j,\ell}-g^-_{j,\ell}\), its vector signed measure is the difference of the vector signed measures for \(g^+_{j,\ell}\) and \(g^-_{j,\ell}\).  Both groups belong to \(\mathcal G_Q\), and coordinatewise total variation is subadditive.
\end{proof}

\subsection{From Multicalibration to Prediction Accuracy}

Assumption~\ref{ass:witness}(ii) gives a lower bound when the preceding predictions equal their true values, but a predictor need not use the true prefix.  The following lemma relates the residual at an arbitrary prediction to the prediction error at the current level and the errors at preceding levels.

\begin{lemma}[Relating Residuals to Prediction Errors]\label{lem:triangular-consequences}
Under Assumption~\ref{ass:witness}, for every \(v\in\mathcal R_0\), every \(p\in\mathcal P\), and every \(j\in[k]\),
\[
\sgn(p_j-v_j)\left(R_j(p_{<j},p_j,\mu_v)-R_j(v_{<j},v_j,\mu_v)\right)
\ge
c_{\mathrm{anti}}|p_j-v_j|
-
C_{\mathrm{prev}}\sum_{i<j}|p_i-v_i|,
\]
and
\[
\left|R_j(p_{<j},p_j,\mu_v)-R_j(v_{<j},v_j,\mu_v)\right|
\le
C_{\mathrm{lip}}|p_j-v_j|
+
C_{\mathrm{prev}}\sum_{i<j}|p_i-v_i|.
\]
\end{lemma}

\begin{proof}
Because \(\Gamma(\mu_v)=v\) by Assumption~\ref{ass:witness}(i), \(R_j(v_{<j},v_j,\mu_v)=0\).  Thus it suffices to prove the same two inequalities with \(R_j(p_{<j},p_j,\mu_v)\) in place of \(R_j(p_{<j},p_j,\mu_v)-R_j(v_{<j},v_j,\mu_v)\).
For the lower bound,
\[
\sgn(p_j-v_j)R_j(p_{<j},p_j,\mu_v)
\ge
\sgn(p_j-v_j)R_j(v_{<j},p_j,\mu_v)
-
\left|R_j(p_{<j},p_j,\mu_v)-R_j(v_{<j},p_j,\mu_v)\right|.
\]
Assumption~\ref{ass:witness}(ii) controls the first term, and condition~(iv) controls the second.  For the upper bound, use the triangle inequality and \(R_j(v_{<j},v_j,\mu_v)=0\), followed by the bounds in conditions~(iii) and~(iv).
\end{proof}

The preceding lemma relates prediction error to a residual multiplied by the sign of the prediction error.  We approximate the threshold signs that arise in this argument by linear combinations of the fixed Walsh functions that define \(\mathcal G_Q\).  Lemma~\ref{lem:signed-weights} shows that \(\Gamma\)-ECE controls residuals weighted by each of these Walsh functions.  The next lemma shows that \(\Gamma\)-ECE also controls a residual weighted by such a linear combination.

\begin{lemma}[Multicalibration Controls Weighted Residuals]\label{lem:variable-walsh-tv}
Fix \(\theta\), a predictor \(\Pi\), and a level \(j\in[k]\).
Suppose a function \(h:\mathcal P\times\mathcal X_Q\to\R\) has the form
\[
h(p,a)
=
\beta_0(p)
+
\sum_{\ell\in\mathcal S_\alpha}
\beta_\ell(p)w_{i,\ell}(a),
\]
for some context coordinate \(i\in[k]\) and measurable coefficient functions
\[
\beta_0:\mathcal P\to\R,
\qquad
\beta_\ell:\mathcal P\to\R
\quad(\ell\in\mathcal S_\alpha),
\]
satisfying
\[
\sup_{p\in\mathcal P}|\beta_0(p)|
+
\sum_{\ell\in\mathcal S_\alpha}
\sup_{p\in\mathcal P}|\beta_\ell(p)|
\le C_{\mathrm{coef}}.
\]
Then
\[
\left|
\frac1{Q^k}
\sum_{a\in\mathcal X_Q}
\int_{\mathcal P}
h(p,a)\,
R_j(p_{<j},p_j,\mu_{v_\theta(a)})\,
\Pi_a(dp)
\right|
\le
2C_{\mathrm{coef}}\,\MCErr^\Gamma_{\mathsf D_\theta}(\Pi;\mathcal G_Q).
\]
\end{lemma}

\begin{proof}
For a signed weight \(w\), the \(j\)-th coordinate signed measure under \((\mathsf D_\theta,\Pi)\) is
\[
\nu^{\mathsf D_\theta,\Pi}_{w,j}(A)
:=
\frac1{Q^k}
\sum_{a\in\mathcal X_Q}
w(a)
\int_A
R_j(p_{<j},p_j,\mu_{v_\theta(a)})\,
\Pi_a(dp),
\qquad A\subseteq\mathcal P.
\]
Therefore
\[
\frac1{Q^k}\sum_a\int h(p,a)R_j(p_{<j},p_j,\mu_{v_\theta(a)})\Pi_a(dp)
=
\int \beta_0(p)\,\nu^{\mathsf D_\theta,\Pi}_{g_{\mathrm{all}},j}(dp)
+
\sum_{\ell\in\mathcal S_\alpha}
\int \beta_\ell(p)\,\nu^{\mathsf D_\theta,\Pi}_{w_{i,\ell},j}(dp).
\]
By total variation,
\[
\left|\int \beta_0\,d\nu^{\mathsf D_\theta,\Pi}_{g_{\mathrm{all}},j}\right|
\le
\sup_p|\beta_0(p)|\,
|\nu^{\mathsf D_\theta,\Pi}_{g_{\mathrm{all}},j}|(\mathcal P)
\le
\sup_p|\beta_0(p)|\,
\MCErr^\Gamma_{\mathsf D_\theta}(\Pi;\mathcal G_Q).
\]
For each signed Walsh weight \(w_{i,\ell}=g^+_{i,\ell}-g^-_{i,\ell}\), Lemma~\ref{lem:signed-weights} gives
\[
|\nu^{\mathsf D_\theta,\Pi}_{w_{i,\ell},j}|(\mathcal P)
\le
\Err^\Gamma_{\mathsf D_\theta}(\Pi;w_{i,\ell})
\le
2\MCErr^\Gamma_{\mathsf D_\theta}(\Pi;\mathcal G_Q).
\]
Therefore
\[
\begin{aligned}
&\left|
\frac1{Q^k}
\sum_a
\int h(p,a)R_j(p_{<j},p_j,\mu_{v_\theta(a)})\Pi_a(dp)
\right|
\\
&\le
\left(
\sup_p|\beta_0(p)|
+
2\sum_{\ell\in\mathcal S_\alpha}\sup_p|\beta_\ell(p)|
\right)
\MCErr^\Gamma_{\mathsf D_\theta}(\Pi;\mathcal G_Q)
\\
&\le
2C_{\mathrm{coef}}\MCErr^\Gamma_{\mathsf D_\theta}(\Pi;\mathcal G_Q).
\end{aligned}
\]
\end{proof}

We now combine the residual inequalities with the group family constructed above.  To establish their lower bound for multicalibration of vector-valued means, \citet{collina2026sample} divide the prediction error into two parts.  Coordinatewise signs control the contribution from predictions that are far from the unperturbed mean for their context, while the all-ones group controls the contribution from predictions that are nearby.  Splitting predictions into these two cases is not enough here, because an error in an earlier coordinate can alter every later residual.  We first use threshold signs to control coordinates \(1,\ldots,k-1\) successively.  With the prefix errors controlled, two-sided threshold signs bound the error in the final coordinate outside the correct box, and the all-ones group handles predictions inside it.  The resulting bound controls the average prediction error by \(\Gamma\)-ECE.

For a predictor \(\Pi=(\Pi_a)_{a\in\mathcal X_Q}\), define its average error in predicting the property under \(\theta\) by
\[
\PredErr_\theta(\Pi)
:=
\frac1{Q^k}\sum_{a\in\mathcal X_Q}
\int_{\mathcal P}
\|p-v_\theta(a)\|_1\,\Pi_a(dp).
\]
Also define the coordinate prediction errors
\[
\PredErr_{\theta,j}(\Pi)
:=
\frac1{Q^k}\sum_{a\in\mathcal X_Q}
\int_{\mathcal P}
|p_j-v_{\theta,j}(a)|\,\Pi_a(dp),
\qquad j\in[k].
\]
Thus
\[
\PredErr_\theta(\Pi)=\sum_{j=1}^k\PredErr_{\theta,j}(\Pi).
\]

\begin{proposition}[\(\Gamma\)-ECE Controls Prediction Error]\label{prop:lower-anticoarsening}
Under Assumption~\ref{ass:witness}, there exists a constant \(C_{\mathrm{pred}}<\infty\), depending only on \(k\), \(\mathcal R_0\), and the constants in Assumption~\ref{ass:witness}, such that for every sufficiently large power of two \(Q\), every \(\theta\in\Theta_Q\), and every randomized predictor \(\Pi\),
\[
\PredErr_\theta(\Pi)
\le
C_{\mathrm{pred}}\log(Q+1)\,
\MCErr^\Gamma_{\mathsf D_\theta}(\Pi;\mathcal G_Q).
\]
\end{proposition}

\begin{proof}
All constants in this proof may depend on \(k\), \(\mathcal R_0\), and the constants in Assumption~\ref{ass:witness}, but not on \(Q\), \(\theta\), or \(\Pi\).  The accuracy \(\alpha\) in the Walsh approximation was fixed small enough that the terms \(C\alpha\PredErr_{\theta,j}(\Pi)\) that arise can be absorbed into the left-hand side.
For brevity, write
\[
e_i(a,p):=|p_i-v_{\theta,i}(a)|,
\qquad
\mathsf R_{a,i}(p):=R_i(p_{<i},p_i,\mu_{v_\theta(a)}).
\]
Because \(\Gamma(\mu_{v_\theta(a)})=v_\theta(a)\) by Assumption~\ref{ass:witness}(i), \(\mathsf R_{a,i}(v_\theta(a))=0\).

\paragraph{Controlling the First \(k-1\) Coordinates.}
For \(j<k\) and \(p_j\in\mathcal P_j\), define
\[
s^{(j)}_{p_j}(a)
:=
\sgn(p_j-b_j(a_j)).
\]
Since \(v_{\theta,j}(a)=b_j(a_j)\) for \(j<k\), this is \(\sgn(p_j-v_{\theta,j}(a))\).  As a function of \(a_j\), it is one of the one-sided threshold signs in \(\mathcal T_Q\).  Let \(\widehat s^{(j)}_{p_j}\) be its Walsh approximation from Lemma~\ref{lem:walsh}.  Lemma~\ref{lem:triangular-consequences} gives
\[
s^{(j)}_{p_j}(a)\mathsf R_{a,j}(p)
\ge
c_{\mathrm{anti}}e_j(a,p)
-
C_{\mathrm{prev}}\sum_{i<j}e_i(a,p),
\]
and also
\[
|\mathsf R_{a,j}(p)|
\le
C e_j(a,p)+C\sum_{i<j}e_i(a,p).
\]
Therefore, using \(\|\widehat s^{(j)}_{p_j}-s^{(j)}_{p_j}\|_\infty\le\alpha\),
\[
\widehat s^{(j)}_{p_j}(a)\mathsf R_{a,j}(p)
\ge
c e_j(a,p)
-
C\sum_{i<j}e_i(a,p)
\]
after fixing \(\alpha\) sufficiently small and decreasing \(c>0\).  Integrating over \((a,p)\) and rearranging gives
\[
c\PredErr_{\theta,j}(\Pi)
\le
\left|
\frac1{Q^k}\sum_a\int
\widehat s^{(j)}_{p_j}(a)\mathsf R_{a,j}(p)\,\Pi_a(dp)
\right|
+
C\sum_{i<j}\PredErr_{\theta,i}(\Pi).
\]
The Walsh coefficients of \(\widehat s^{(j)}_{p_j}(a)\), viewed as a function of \(a\), are step functions of \(p_j\) and hence measurable.  They also satisfy the uniform bound \(C\log(Q+1)\).  Lemma~\ref{lem:variable-walsh-tv} therefore bounds the term inside the absolute value by
\[
C\log(Q+1)\MCErr^\Gamma_{\mathsf D_\theta}(\Pi;\mathcal G_Q),
\]
and hence
\[
\PredErr_{\theta,j}(\Pi)
\le
C\log(Q+1)\MCErr^\Gamma_{\mathsf D_\theta}(\Pi;\mathcal G_Q)
+
C\sum_{i<j}\PredErr_{\theta,i}(\Pi).
\]
Induction over \(j=1,\ldots,k-1\) yields
\begin{equation}\label{eq:prefix-coordinate-bound}
\sum_{j<k}\PredErr_{\theta,j}(\Pi)
\le
C\log(Q+1)\MCErr^\Gamma_{\mathsf D_\theta}(\Pi;\mathcal G_Q).
\end{equation}

\paragraph{Predictions Outside the Boxes.}
Set
\[
r:=5\eta.
\]
Since \(c_\eta<1/20\), we have \(r<d_Q/4\).
Define boxes around the unperturbed property vectors by
\[
C_a
:=
\{p\in\mathcal P:\ |p_i-b_i(a_i)|\le r\ \text{for every }i\in[k]\}.
\]
Each true property vector \(v_\theta(a)\) lies in \(C_a\), and the boxes are pairwise disjoint because \(2r<d_Q\).  Let
\[
E_k^{\mathrm{out}}
:=
\frac1{Q^k}
\sum_a
\int_{\mathcal P\setminus C_a}
|p_k-v_{\theta,k}(a)|\,\Pi_a(dp).
\]
Decompose this error according to whether the final prediction is near \(b_k(a_k)\).  Define
\[
\begin{aligned}
E_k^{\mathrm{out,near}}
&:=
\frac1{Q^k}
\sum_a
\int_{\mathcal P\setminus C_a}
\mathbf 1\{|p_k-b_k(a_k)|\le r\}
|p_k-v_{\theta,k}(a)|\,\Pi_a(dp),
\\
E_k^{\mathrm{out,far}}
&:=
\frac1{Q^k}
\sum_a
\int_{\mathcal P\setminus C_a}
\mathbf 1\{|p_k-b_k(a_k)|>r\}
|p_k-v_{\theta,k}(a)|\,\Pi_a(dp).
\end{aligned}
\]
Then
\[
E_k^{\mathrm{out}}
=
E_k^{\mathrm{out,near}}
+
E_k^{\mathrm{out,far}}.
\]
On the near part, \(p\notin C_a\) implies that some prefix coordinate \(i<k\) has \(|p_i-b_i(a_i)|>r\), while
\[
|p_k-v_{\theta,k}(a)|
\le r+\eta
\le C\sum_{i<k}|p_i-v_{\theta,i}(a)|.
\]
The second inequality uses \(v_{\theta,i}(a)=b_i(a_i)\) for \(i<k\), and the fact that one prefix error is at least \(r\) while \(r\) and \(\eta\) are fixed constant multiples of one another.
Therefore
\[
E_k^{\mathrm{out,near}}
\le
C\sum_{i<k}\PredErr_{\theta,i}(\Pi).
\]

For the far part, define the two-sided threshold sign
\[
s^{(k)}_{p_k}(a)
:=
\frac12
\left(
\sgn(p_k-r-b_k(a_k))
+
\sgn(p_k+r-b_k(a_k))
\right).
\]
As a function of \(a_k\), this belongs to \(\mathcal T_Q\).  On the event \(|p_k-b_k(a_k)|>r\), it agrees with \(\sgn(p_k-v_{\theta,k}(a))\), because \(|v_{\theta,k}(a)-b_k(a_k)|=\eta<r\).  Lemma~\ref{lem:triangular-consequences} gives
\[
c_{\mathrm{anti}}\mathbf 1\{|p_k-b_k(a_k)|>r\}|p_k-v_{\theta,k}(a)|
\le
s^{(k)}_{p_k}(a)\mathsf R_{a,k}(p)
+
C\sum_{i<k}|p_i-v_{\theta,i}(a)|.
\]
Let \(\widehat s^{(k)}_{p_k}\) be the Walsh approximation to \(s^{(k)}_{p_k}\).  Replacing \(s^{(k)}_{p_k}\) by \(\widehat s^{(k)}_{p_k}\), integrating, and controlling the approximation error gives
\[
\begin{aligned}
c_{\mathrm{anti}}E_k^{\mathrm{out,far}}
&\le
\left|
\frac1{Q^k}\sum_a\int
\widehat s^{(k)}_{p_k}(a)\mathsf R_{a,k}(p)\,\Pi_a(dp)
\right|
\\
&\quad
+
C\sum_{i<k}\PredErr_{\theta,i}(\Pi)
+
C\alpha\PredErr_{\theta,k}(\Pi).
\end{aligned}
\]
Indeed, the extra \(C\alpha\PredErr_{\theta,k}(\Pi)\) term comes from
\[
\left|\widehat s^{(k)}_{p_k}(a)-s^{(k)}_{p_k}(a)\right|\,|\mathsf R_{a,k}(p)|
\le
\alpha\left(
C e_k(a,p)+C\sum_{i<k}e_i(a,p)
\right),
\]
where the residual bound is the upper inequality in Lemma~\ref{lem:triangular-consequences}.  The Walsh coefficients again satisfy the conditions of Lemma~\ref{lem:variable-walsh-tv}, with coefficient bound \(C\log(Q+1)\), so that lemma bounds the term inside the absolute value by
\[
C\log(Q+1)\MCErr^\Gamma_{\mathsf D_\theta}(\Pi;\mathcal G_Q).
\]
Therefore
\[
E_k^{\mathrm{out,far}}
\le
C\log(Q+1)\MCErr^\Gamma_{\mathsf D_\theta}(\Pi;\mathcal G_Q)
+
C\sum_{i<k}\PredErr_{\theta,i}(\Pi)
+
C\alpha\PredErr_{\theta,k}(\Pi).
\]
Combining the near and far bounds with \eqref{eq:prefix-coordinate-bound},
\begin{equation}\label{eq:final-outside-bound}
E_k^{\mathrm{out}}
\le
C\log(Q+1)\MCErr^\Gamma_{\mathsf D_\theta}(\Pi;\mathcal G_Q)
+
C\alpha\PredErr_{\theta,k}(\Pi).
\end{equation}

\paragraph{Predictions Inside the Boxes.}
Inside the boxes, we use the all-ones group rather than another threshold sign.  Because the boxes \(C_a\) are pairwise disjoint, the residual contributions localized to different contexts occupy disjoint regions of the prediction space and therefore do not cancel in total variation.  The all-ones residual also includes predictions outside the boxes, whose contribution is controlled by the preceding bounds.

Let
\[
E_k^{\mathrm{loc}}
:=
\frac1{Q^k}
\sum_a
\int_{C_a}
|p_k-v_{\theta,k}(a)|\,\Pi_a(dp),
\]
so that
\[
\PredErr_{\theta,k}(\Pi)
=
E_k^{\mathrm{loc}}+E_k^{\mathrm{out}}.
\]
Let \(\nu^{\mathsf D_\theta,\Pi}_{g_{\mathrm{all}},k}\) be the signed measure for the final coordinate and the all-ones group, and define its localized part by
\[
\nu^{\mathrm{loc}}_k(A)
:=
\frac1{Q^k}
\sum_a
\int_{A\cap C_a}
\mathsf R_{a,k}(p)\,\Pi_a(dp).
\]
Because the boxes \(C_a\) are pairwise disjoint,
\[
|\nu^{\mathrm{loc}}_k|(\mathcal P)
=
\frac1{Q^k}
\sum_a
\int_{C_a}
|\mathsf R_{a,k}(p)|\,\Pi_a(dp).
\]
The lower inequality in Lemma~\ref{lem:triangular-consequences} implies, pointwise,
\[
c_{\mathrm{anti}}e_k(a,p)
\le
|\mathsf R_{a,k}(p)|
+
C\sum_{i<k}e_i(a,p).
\]
Integrating this bound over \(C_a\), summing over \(a\), and dividing by \(Q^k\) gives
\[
E_k^{\mathrm{loc}}
\le
C|\nu^{\mathrm{loc}}_k|(\mathcal P)
+
C\sum_{i<k}\PredErr_{\theta,i}(\Pi).
\]
Write
\[
\nu^{\mathsf D_\theta,\Pi}_{g_{\mathrm{all}},k}
=
\nu^{\mathrm{loc}}_k+\nu^{\mathrm{out}}_k,
\]
where \(\nu^{\mathrm{out}}_k\) is the contribution from \(p\notin C_a\).  Lemma~\ref{lem:triangular-consequences} gives
\[
|\nu^{\mathrm{out}}_k|(\mathcal P)
\le
C E_k^{\mathrm{out}}
+
C\sum_{i<k}\PredErr_{\theta,i}(\Pi).
\]
Here the total variation of the outside contribution is bounded by the integral of \(|\mathsf R_{a,k}(p)|\) over \(p\notin C_a\); the boxes \(C_a\) are disjoint, but their complements need not be.  The upper inequality in Lemma~\ref{lem:triangular-consequences} then bounds this integral by \(E_k^{\mathrm{out}}\) and the prefix errors.
Since
\[
|\nu^{\mathsf D_\theta,\Pi}_{g_{\mathrm{all}},k}|(\mathcal P)
\le
\MCErr^\Gamma_{\mathsf D_\theta}(\Pi;\mathcal G_Q),
\]
the measure decomposition also gives
\[
|\nu^{\mathrm{loc}}_k|(\mathcal P)
\le
|\nu^{\mathsf D_\theta,\Pi}_{g_{\mathrm{all}},k}|(\mathcal P)
+
|\nu^{\mathrm{out}}_k|(\mathcal P).
\]
Combining these bounds gives
\[
E_k^{\mathrm{loc}}
\le
C\MCErr^\Gamma_{\mathsf D_\theta}(\Pi;\mathcal G_Q)
+
C E_k^{\mathrm{out}}
+
C\sum_{i<k}\PredErr_{\theta,i}(\Pi).
\]
Using \eqref{eq:prefix-coordinate-bound} and \eqref{eq:final-outside-bound},
\[
\PredErr_{\theta,k}(\Pi)
\le
C\log(Q+1)\MCErr^\Gamma_{\mathsf D_\theta}(\Pi;\mathcal G_Q)
+
C\alpha\PredErr_{\theta,k}(\Pi).
\]
Absorbing the final term into the left-hand side gives
\[
\PredErr_{\theta,k}(\Pi)
\le
C\log(Q+1)\MCErr^\Gamma_{\mathsf D_\theta}(\Pi;\mathcal G_Q).
\]
Together with \eqref{eq:prefix-coordinate-bound}, this proves the proposition.
\end{proof}

\subsection{Recovering the Sign Assignment}

Proposition~\ref{prop:lower-anticoarsening} converts small \(\Gamma\)-ECE into small average prediction error.  Distinct sign assignments produce different true property vectors on a constant fraction of the contexts, so sufficiently small prediction error identifies the true assignment.  The KL bound in Assumption~\ref{ass:witness}(v) limits how much information \(n\) samples contain about that assignment, and Fano's inequality then gives the required lower bound on sample complexity.

\begin{lemma}[Different Sign Assignments Produce Separated Property Vectors]\label{lem:lower-separation}
For every distinct \(\theta,\theta'\in\Theta_Q\),
\[
\frac1{Q^k}
\sum_{a\in\mathcal X_Q}
\|v_\theta(a)-v_{\theta'}(a)\|_1
\ge
2\rho_{\mathrm{pack}}\eta.
\]
\end{lemma}

\begin{proof}
For every context \(a\), the vectors \(v_\theta(a)\) and \(v_{\theta'}(a)\) agree in their first \(k-1\) coordinates.  Whenever \(\theta_a\neq\theta'_a\), their final coordinates satisfy
\[
|v_{\theta,k}(a)-v_{\theta',k}(a)|=2\eta.
\]
The packing condition gives at least \(\rho_{\mathrm{pack}} Q^k\) such contexts.
\end{proof}

\begin{lemma}[Prediction Error Implies Exact Decoding]\label{lem:lower-decoding}
There exists a decoder \(\widehat\theta(\Pi)\) such that, if
\[
\PredErr_\theta(\Pi)\le\frac{\rho_{\mathrm{pack}}\eta}{2},
\]
then
\[
\widehat\theta(\Pi)=\theta.
\]
\end{lemma}

\begin{proof}
For each \(a\), define the mean prediction value
\[
\bar p_\Pi(a)
:=
\int_{\mathcal P}p\,\Pi_a(dp).
\]
By Jensen's inequality,
\[
\frac1{Q^k}\sum_a
\|\bar p_\Pi(a)-v_\theta(a)\|_1
\le
\PredErr_\theta(\Pi).
\]
Decode by nearest neighbor:
\[
\widehat\theta(\Pi)
\in
\arg\min_{\theta'\in\Theta_Q}
\frac1{Q^k}\sum_a
\|\bar p_\Pi(a)-v_{\theta'}(a)\|_1.
\]
If \(\theta'\neq\theta\), then Lemma~\ref{lem:lower-separation} and the triangle inequality imply
\[
\frac1{Q^k}\sum_a
\|\bar p_\Pi(a)-v_{\theta'}(a)\|_1
\ge
2\rho_{\mathrm{pack}}\eta-\frac{\rho_{\mathrm{pack}}\eta}{2}
>
\frac{\rho_{\mathrm{pack}}\eta}{2}
\ge
\frac1{Q^k}\sum_a
\|\bar p_\Pi(a)-v_\theta(a)\|_1.
\]
Thus the unique nearest neighbor is \(\theta\).
\end{proof}

\begin{lemma}[Pairwise KL Bound]\label{lem:lower-kl}
For every \(\theta,\theta'\in\Theta_Q\) and every sample size \(n\),
\[
\KL(\mathsf D_\theta^n\,\|\,\mathsf D_{\theta'}^n)
\le
4C_{\mathrm{KL}}\,n\eta^2.
\]
\end{lemma}

\begin{proof}
Since \(X\) is uniform on \(\mathcal X_Q\) under both distributions,
\[
\KL(\mathsf D_\theta\,\|\,\mathsf D_{\theta'})
=
\frac1{Q^k}
\sum_a
\KL\left(
\mu_{v_\theta(a)}
\,\middle\|\,
\mu_{v_{\theta'}(a)}
\right).
\]
By Assumption~\ref{ass:witness}(v),
\[
\KL\left(
\mu_{v_\theta(a)}
\,\middle\|\,
\mu_{v_{\theta'}(a)}
\right)
\le
C_{\mathrm{KL}}
\|v_\theta(a)-v_{\theta'}(a)\|_2^2
\le
4C_{\mathrm{KL}}\eta^2.
\]
Tensorization over \(n\) i.i.d. samples gives the claim.
\end{proof}

Combining Proposition~\ref{prop:lower-anticoarsening}, Lemma~\ref{lem:lower-decoding}, and Lemma~\ref{lem:lower-kl} gives a lower bound on sample complexity at a fixed resolution.

\begin{proposition}[Lower Bound at Resolution \(Q\)]\label{prop:lower-resolution}
There are constants \(c_*,C_*>0\), depending only on \(k\), \(\mathcal R_0\), and the constants in Assumption~\ref{ass:witness}, such that the following holds for all sufficiently large powers of two \(Q\).  If a possibly randomized learner, given \(n\) i.i.d. samples from \(\mathsf D_\theta\), outputs a predictor \(\Pi\) satisfying
\[
\MCErr^\Gamma_{\mathsf D_\theta}(\Pi;\mathcal G_Q)
\le
c_*\frac{\eta}{\log(Q+1)}
\]
with probability at least \(2/3\) for every \(\theta\in\Theta_Q\), then
\[
n\ge C_*Q^{k+2}.
\]
\end{proposition}

\begin{proof}
Choose \(c_*>0\) small enough that Proposition~\ref{prop:lower-anticoarsening} implies
\[
\PredErr_\theta(\Pi)\le\frac{\rho_{\mathrm{pack}}\eta}{2}
\]
whenever the learner achieves the multicalibration error required in the proposition.  By Lemma~\ref{lem:lower-decoding}, the learner then induces an exact decoder for \(\theta\) with success probability at least \(2/3\).

Let \(\Theta\) be uniformly distributed over \(\Theta_Q\), and let the sample be drawn from \(\mathsf D_\Theta^n\).  If the learner is randomized, include its independent random seed in the decoder's observation; this does not change the KL bound.  Fano's inequality~\citep{CoverThomas2006} gives
\[
\Prob[\widehat\Theta\neq\Theta]
\ge
1-
\frac{I(\Theta;\text{sample})+\log 2}{\log|\Theta_Q|}.
\]
The standard bound on mutual information in terms of pairwise KL divergence, together with Lemma~\ref{lem:lower-kl}, implies
\[
I(\Theta;\text{sample})
\le
4C_{\mathrm{KL}}n\eta^2.
\]
A decoder with error probability at most \(1/3\) therefore requires
\[
4C_{\mathrm{KL}}n\eta^2+\log 2
\ge
\frac23\log|\Theta_Q|
\ge
cQ^k
\]
for a constant \(c>0\).  For all sufficiently large \(Q\), the \(\log2\) term is absorbed into the right-hand side, so
\[
4C_{\mathrm{KL}}n\eta^2
\ge
cQ^k
\]
after decreasing \(c\).  Since \(\eta\asymp Q^{-1}\),
\[
n\ge C_*Q^{k+2}.
\]
\end{proof}

\subsection{Proof of the Main Theorem}

Choosing \(Q\) as a function of the target error \(\eps\) converts Proposition~\ref{prop:lower-resolution} into Theorem~\ref{thm:lower-bound}.

\begin{proof}[Proof of Theorem~\ref{thm:lower-bound}]
For each large power of two \(Q\), Proposition~\ref{prop:lower-resolution} gives a finite collection of candidate distributions for which \(n\ge C_*Q^{k+2}\) samples are required at error level
\[
\eps_Q
:=
c_*\frac{\eta}{\log(Q+1)}
\asymp
\frac1{Q\log Q}.
\]
Given sufficiently small \(\eps>0\), choose \(Q\) to be the largest power of two such that
\[
\eps
\le
\eps_Q.
\]
Then
\[
Q
\ge
c'\frac1{\eps\log(C'/\eps)}
\]
for constants \(c',C'>0\).  Therefore
\[
n
\ge
C_*Q^{k+2}
\ge
c\frac{\eps^{-(k+2)}}{\log^{k+2}(C/\eps)}
\]
for constants \(c,C>0\) depending only on \(k\), \(\kappa\), \(\mathcal R_0\), and the constants in Assumption~\ref{ass:witness}.
The group family satisfies
\[
|\mathcal G_Q|
\le
C_{\alpha,k}\log^3(Q+1)
\le
\eps^{-\kappa}
\]
for every fixed \(\kappa>0\) and all sufficiently small \(\eps\).  Taking
\[
\mathcal X_\eps:=\mathcal X_Q,
\qquad
\mathcal G_\eps:=\mathcal G_Q,
\qquad
\mathfrak D_\eps:=\{\mathsf D_\theta:\theta\in\Theta_Q\}
\]
proves the theorem.
\end{proof}

\section{The Upper Bound}\label{sec:upper}

Our route to the upper bound is through an online forecasting problem.  On each round, the forecaster constructs a prediction rule from the preceding rounds, applies it to the current context, and then observes the outcome.  We first control the empirical multicalibration error accumulated over these rounds.

Given a dataset consisting of \(T\) i.i.d. context--outcome pairs, we run the online procedure for \(T\) rounds, using one pair on each round.  The context component of the \(t\)-th pair is revealed before the prediction and its outcome component afterward.  The procedure produces \(T\) prediction rules.  We return their average as our predictor.  A martingale argument then transfers the empirical guarantee for the online transcript to a population multicalibration guarantee for this predictor.

We begin by formulating the online problem and expressing its empirical multicalibration error as a collection of objectives that the forecaster must control simultaneously.  We then state the regularity conditions, construct and analyze the forecaster, and carry out the online-to-batch reduction.  Finally, we show how to implement the forecaster efficiently using linear optimization over \(\mathcal M\).

\subsection{The Online Formulation and Its Objectives}

We first describe the online formulation for an arbitrary finite set \(\mathcal P_Q\subset\mathcal P\) of possible predictions.  Fix a horizon \(T\) and a finite group family \(\mathcal G\).

Let \(\mathcal F_{t-1}\) be the sigma-field generated by the observations through round \(t-1\) and any internal randomness used by the forecaster before round \(t\).  At the beginning of round \(t\), the forecaster uses this history to define an \(\mathcal F_{t-1}\)-measurable rule \(\pi_t:\mathcal X\to\Delta(\mathcal P_Q)\).  After the context \(X_t\) is revealed, it outputs \(\pi_t(\cdot\mid X_t)\) as the randomized prediction and then observes the outcome \(Y_t\).

For any sequence of such rules, define the empirical multicalibration error over the \(T\) online rounds by
\[
\widehat{\MCErr}^{\Gamma}_T
:=
\frac1T
\max_{g\in\mathcal G}
\sum_{p\in\mathcal P_Q}
\left\|
\sum_{t=1}^T
 g(X_t)\pi_t(p\mid X_t)R(p,Y_t)
\right\|_1.
\]
Let \(\Sigma:=\{\pm1\}^{\mathcal P_Q\times[k]}\), and write \(s_{p,j}\) for the sign assigned by \(s\in\Sigma\) to the pair \((p,j)\).  Representing each absolute value in the \(\ell_1\)-norm as a maximum over its sign gives
\[
\begin{aligned}
\widehat{\MCErr}^{\Gamma}_T
&=
\frac1T
\max_{\substack{g\in\mathcal G\\s\in\Sigma}}
\sum_{p\in\mathcal P_Q}
\sum_{j=1}^k
s_{p,j}
\sum_{t=1}^T
g(X_t)\pi_t(p\mid X_t)
R_j(p_{<j},p_j,Y_t)
\\
&=
\frac1T
\max_{\substack{g\in\mathcal G\\s\in\Sigma}}
\sum_{t=1}^T
g(X_t)
\E_{p\sim\pi_t(\cdot\mid X_t)}
\left[
\sum_{j=1}^k
s_{p,j}R_j(p_{<j},p_j,Y_t)
\right].
\end{aligned}
\]
The above equalities express the empirical error as the largest cumulative signed residual among the objectives indexed by \((g,s)\in\mathcal G\times\Sigma\).  Thus, the online task is a multiobjective learning problem: the forecaster must control all of these signed objectives simultaneously~\citep{lee2022online}.  We use the framework of online learning with expert advice.  Each pair \((g,s)\) is treated as an expert, whose gain on round \(t\) is
\[
h_t(g,s)
:=
g(X_t)
\E_{p\sim\pi_t(\cdot\mid X_t)}
\left[
\sum_{j=1}^k
s_{p,j}R_j(p_{<j},p_j,Y_t)
\right].
\]
Since the empirical error is the largest cumulative gain of any expert divided by \(T\), our goal is to control the cumulative gain of every expert.  The exponential weights algorithm chooses a distribution \(w_t\) over the experts on each round.  We call the expectation of \(h_t(g,s)\) under this distribution the expected gain under \(w_t\).  Its regret guarantee bounds the cumulative gain of every individual expert by the sum of the expected gains under \(w_t\), plus a sublinear regret term.  It therefore remains to keep these expected gains small.

We use the same weights \(w_t\) to choose the forecaster's distribution on \(\mathcal P_Q\).  For any candidate distribution \(\pi\in\Delta(\mathcal P_Q)\) and any admissible outcome law, we can evaluate the expected gain under \(w_t\) from the corresponding signed residuals.  The forecaster chooses \(\pi\) in a minimax manner, minimizing the largest possible value of this expected gain over admissible outcome laws.  We allow an additive error \(\rho\ge0\) in this minimization.  \citet{noarov2025high} use a related construction for high-dimensional unbiased prediction, in which an online learning algorithm weights signed objectives and those weights guide the randomized forecast.

\subsection{Regularity Assumptions}

The online formulation can be stated for any finite set \(\mathcal P_Q\), but its analysis requires additional structure.  Convexity, compactness, and continuity (Condition~(i) below) allow us to apply Sion's minimax theorem.  Although the forecaster must choose a distribution on \(\mathcal P_Q\) without knowing the conditional outcome law, the theorem allows us to analyze the one-round objective by first fixing that law and then choosing the distribution.  A uniform bound on the residual (Condition~(ii)) controls both the expert gains and the martingale deviations.  Finally, Condition~(iii) places two requirements on the grid.  It has only \(O(Q^k)\) points, and for every admissible outcome law \(\mu\), some \(p\in\mathcal P_Q\) makes every coordinate of \(R(p,\mu)\) small.  After \(\mu\) is fixed in the reversed analysis, choosing the point mass at such a \(p\) makes the expected gain under \(w_t\) small, which bounds the one-round minimax value.  The following assumption makes these requirements precise.

\begin{assumption}[Regularity for the Upper Bound]\label{ass:upper-analytic}
There are constants \(R_{\max},C_{\mathrm{grid}}<\infty\).  For each integer \(Q\ge1\), there is a finite grid \(\mathcal P_Q\subset\mathcal P\) such that the following hold.
\begin{enumerate}[label=(\roman*)]
\item \(\mathcal M\) is a convex compact subset of a topological vector space of finite signed measures on \(\mathcal Y\), and for every \(p\in\mathcal P_Q\) and coordinate \(j\in[k]\), the map
\[
\mu\mapsto R_j(p_{<j},p_j,\mu)
\]
is affine and continuous on \(\mathcal M\).
\item For every \(p\in\mathcal P\) and \(y\in\mathcal Y\),
\[
\|R(p,y)\|_\infty\le R_{\max}.
\]
\item The grid size satisfies
\[
|\mathcal P_Q|\le C_{\mathrm{grid}}Q^k,
\]
and the expected residual can be rounded at rate \(1/Q\):
\[
\Delta_Q
:=
\sup_{\mu\in\mathcal M}
\min_{p\in\mathcal P_Q}
\|R(p,\mu)\|_\infty
\le
\frac{C_{\mathrm{grid}}}{Q}.
\]
\end{enumerate}
\end{assumption}

\begin{remark}[A Sufficient Lipschitz Condition]
A convenient way to ensure Assumption~\ref{ass:upper-analytic}(iii) is to require a constant \(L_R\) such that
\[
\|R(p,\mu)-R(p',\mu)\|_\infty
\le L_R\|p-p'\|_\infty
\qquad
\forall \mu\in\mathcal M,
\quad p,p'\in\mathcal P,
\]
together with a finite grid \(\mathcal P_Q\subset\mathcal P\) whose \(\ell_\infty\)-mesh is \(O(1/Q)\).  Because the residual vanishes at the true property, every \(\Gamma(\mu)\) can then be rounded to some \(p\in\mathcal P_Q\) with \(\|R(p,\mu)\|_\infty=O(1/Q)\).
\end{remark}

\subsection{The Online Forecaster and Its Empirical Guarantee}\label{sec:upper-algorithm}

Fix an integer \(Q\ge1\), and let \(\mathcal P_Q\) be the grid supplied by Assumption~\ref{ass:upper-analytic}.  Algorithm~\ref{alg:upper-online} specifies how the forecaster operates on each round.

\begin{algorithm}[H]
\caption{Online Forecaster}\label{alg:upper-online}
\begin{algorithmic}[1]
\Require Horizon \(T\), group family \(\mathcal G\), grid \(\mathcal P_Q\), law class \(\mathcal M\), level count \(k\), additive optimization error \(\rho\)
\State Set
\[
\eta_{\mathrm{EW}}\gets
\frac{\sqrt{2\log|\mathcal G\times\Sigma|}}{kR_{\max}\sqrt T}.
\]
\State Initialize cumulative gains \(C_0(g,s)\gets0\) for all \((g,s)\in\mathcal G\times\Sigma\).
\For{\(t=1,\ldots,T\)}
\State Compute
\[
w_t(g,s)
\gets
\frac{\exp\{\eta_{\mathrm{EW}}C_{t-1}(g,s)\}}
{\sum_{(g',s')\in\mathcal G\times\Sigma}\exp\{\eta_{\mathrm{EW}}C_{t-1}(g',s')\}}.
\]
\State Define \(\pi_t(\cdot\mid x)\) pointwise for \(x\in\mathcal X\) to satisfy
\begin{equation}\label{eq:upper-minimax}
\begin{aligned}
&\max_{\mu\in\mathcal M}
\E_{p\sim\pi_t(\cdot\mid x)}\left[
\sum_{g\in\mathcal G}
\sum_{s\in\Sigma}
w_t(g,s)g(x)
\sum_{j=1}^ks_{p,j}R_j(p_{<j},p_j,\mu)
\right]
\\
&\qquad\le
\min_{\pi\in\Delta(\mathcal P_Q)}
\max_{\mu\in\mathcal M}
\E_{p\sim\pi}\left[
\sum_{g\in\mathcal G}
\sum_{s\in\Sigma}
w_t(g,s)g(x)
\sum_{j=1}^ks_{p,j}R_j(p_{<j},p_j,\mu)
\right]
+\rho.
\end{aligned}
\end{equation}
\State Observe \(X_t\) and output \(\pi_t(\cdot\mid X_t)\) as the randomized prediction.
\State Observe \(Y_t\).
\State For every \((g,s)\in\mathcal G\times\Sigma\), set
\[
h_t(g,s)
\gets
g(X_t)
\E_{p\sim\pi_t(\cdot\mid X_t)}
\left[
\sum_{j=1}^ks_{p,j}R_j(p_{<j},p_j,Y_t)
\right].
\]
\State Set \(C_t(g,s)\gets C_{t-1}(g,s)+h_t(g,s)\) for every \((g,s)\in\mathcal G\times\Sigma\).
\EndFor
\Ensure Prediction rules \(\pi_1,\ldots,\pi_T\).
\end{algorithmic}
\end{algorithm}

The prediction rules in Algorithm~\ref{alg:upper-online} can be chosen jointly measurable, as shown in Appendix~\ref{app:upper-measurability}.

The forecaster chooses \(\pi_t\) to keep the expected gain under \(w_t\) small on each round, while the no-regret guarantee of the exponential weights algorithm bounds the cumulative gain of every expert by the sum of these expected gains plus a sublinear regret term.  Since the empirical multicalibration error is the largest cumulative expert gain divided by \(T\), this gives the following guarantee.

\begin{theorem}[Empirical Guarantee for the Online Forecaster]\label{thm:upper-online}
Suppose Assumption~\ref{ass:upper-analytic} holds.  Let \((X_t,Y_t)_{t=1}^T\) be any stochastic process such that, conditionally on \(\mathcal F_{t-1}\) and \(X_t\), the conditional law of \(Y_t\) belongs to \(\mathcal M\).  Then Algorithm~\ref{alg:upper-online} satisfies
\[
\E\widehat{\MCErr}^{\Gamma}_T
\le
\rho+k\Delta_Q
+
C_k R_{\max}\sqrt{\frac{\log|\mathcal G|+k|\mathcal P_Q|}{T}},
\]
where \(C_k\) depends only on \(k\).  Consequently, the two bounds in Assumption~\ref{ass:upper-analytic}(iii) give
\begin{equation}\label{eq:upper-empirical-bound}
\E\widehat{\MCErr}^{\Gamma}_T
\le
\rho
+
O\!\left(
\frac1Q+R_{\max}\sqrt{\frac{Q^k+\log|\mathcal G|}{T}}
\right).
\end{equation}
The expectation is over the stochastic transcript and any internal randomization of the forecaster.
\end{theorem}

We prove Theorem~\ref{thm:upper-online} in the next subsection.

\subsection{Analysis of the Online Forecaster}\label{sec:upper-online-forecaster}

Throughout the analysis, \(\pi_t\) denotes the \(\mathcal F_{t-1}\)-measurable rule selected at round \(t\).  Since every group function takes values in \([0,1]\) and the residual bound in Assumption~\ref{ass:upper-analytic}(ii) gives \(\|R\|_\infty\le R_{\max}\), every signed gain satisfies
\[
|h_t(g,s)|\le kR_{\max}.
\]
By the signed-objective representation of the empirical multicalibration error,
\[
\widehat{\MCErr}^{\Gamma}_T
=
\frac1T
\max_{(g,s)\in\mathcal G\times\Sigma}
\sum_{t=1}^T h_t(g,s).
\]
Thus controlling the gains of all experts \((g,s)\in\mathcal G\times\Sigma\) controls empirical \(\Gamma\)-ECE.

\subsubsection{Exponential Weights}

The first step compares the largest cumulative signed residual with the sum of the expected gains under the distributions \(w_t\) produced by exponential weights.

\begin{lemma}[Exponential Weights Bound]\label{lem:upper-exponential-weights}
For every realized transcript of Algorithm~\ref{alg:upper-online},
\[
\max_{(g,s)\in\mathcal G\times\Sigma}\sum_{t=1}^T h_t(g,s)
\le
\sum_{t=1}^T\sum_{g\in\mathcal G}\sum_{s\in\Sigma}w_t(g,s)h_t(g,s)
+
C_k R_{\max}\sqrt{T\bigl(\log|\mathcal G|+k|\mathcal P_Q|\bigr)},
\]
where \(C_k\) depends only on \(k\) and \(h_t(g,s)\) is the signed gain from Algorithm~\ref{alg:upper-online},
\[
h_t(g,s)
=
g(X_t)
\E_{p\sim\pi_t(\cdot\mid X_t)}
\left[
\sum_{j=1}^ks_{p,j}R_j(p_{<j},p_j,Y_t)
\right].
\]
\end{lemma}

\begin{proof}
Using the learning rate \(\eta_{\mathrm{EW}}\) from Algorithm~\ref{alg:upper-online}, define the exponential weights potential
\[
W_t
:=
\sum_{g\in\mathcal G}\sum_{s\in\Sigma}
\exp\{\eta_{\mathrm{EW}} C_{t-1}(g,s)\}.
\]
Since \(C_t(g,s)=C_{t-1}(g,s)+h_t(g,s)\), and since \(|h_t(g,s)|\le kR_{\max}\), Hoeffding's lemma~\citep{Hoeffding1963} gives
\[
\log\frac{W_{t+1}}{W_t}
=
\log\sum_{g\in\mathcal G}\sum_{s\in\Sigma}
w_t(g,s)\exp\{\eta_{\mathrm{EW}} h_t(g,s)\}
\le
\eta_{\mathrm{EW}}\sum_{g\in\mathcal G}\sum_{s\in\Sigma}w_t(g,s)h_t(g,s)
+\frac{\eta_{\mathrm{EW}}^2k^2R_{\max}^2}{2}.
\]
Summing over \(t\) yields
\[
\log W_{T+1}
\le
\log|\mathcal G\times\Sigma|
+
\eta_{\mathrm{EW}}\sum_{t=1}^T\sum_{g\in\mathcal G}\sum_{s\in\Sigma}
w_t(g,s)h_t(g,s)
+\frac{\eta_{\mathrm{EW}}^2k^2R_{\max}^2T}{2}.
\]
On the other hand, for every \((g,s)\in\mathcal G\times\Sigma\),
\[
\log W_{T+1}\ge \eta_{\mathrm{EW}} C_T(g,s)
=
\eta_{\mathrm{EW}}\sum_{t=1}^T h_t(g,s).
\]
Combining the two inequalities and maximizing over \((g,s)\) gives
\[
\max_{(g,s)\in\mathcal G\times\Sigma}\sum_{t=1}^T h_t(g,s)
\le
\sum_{t=1}^T\sum_{g\in\mathcal G}\sum_{s\in\Sigma}w_t(g,s)h_t(g,s)
+
\frac{\log|\mathcal G\times\Sigma|}{\eta_{\mathrm{EW}}}
+\frac{\eta_{\mathrm{EW}} k^2R_{\max}^2T}{2}.
\]
The expert class satisfies
\[
\log|\mathcal G\times\Sigma|
=
\log|\mathcal G|+k|\mathcal P_Q|\log2.
\]
Substituting this expression and the value of \(\eta_{\mathrm{EW}}\) from Algorithm~\ref{alg:upper-online} into the regret terms completes the proof.
\end{proof}

\subsubsection{The Minimax Value}

Conditional on the history and the current context, the expected gain under \(w_t\) is obtained by evaluating the objective in \eqref{eq:upper-minimax} at the conditional law of the outcome.  By the choice of \(\pi_t\), it is bounded by the worst-case value of the one-round minimax problem, up to the additive accuracy \(\rho\).  We next bound this value.

\begin{lemma}[Upper Bound on the Minimax Value]\label{lem:upper-minimax-value}
Under Assumption~\ref{ass:upper-analytic}, the minimax term on the right-hand side of \eqref{eq:upper-minimax} is at most \(k\Delta_Q\) for every history and every context \(x\).
\end{lemma}

\begin{proof}
Fix the history and the context \(x\), and let \(F(\pi,\mu)\) denote the objective in \eqref{eq:upper-minimax}.  We first bound the value when the outcome law is chosen before the prediction:
\[
\max_{\mu\in\mathcal M}
\min_{\pi\in\Delta(\mathcal P_Q)}
F(\pi,\mu).
\]
Fix \(\mu\in\mathcal M\).  Assumption~\ref{ass:upper-analytic}(iii) provides \(p_\mu\in\mathcal P_Q\) such that
\[
\|R(p_\mu,\mu)\|_\infty\le\Delta_Q.
\]
Choose the point mass \(\pi=\delta_{p_\mu}\) at \(p_\mu\).  For every \((g,s)\), the bound \(0\le g(x)\le1\) gives
\[
 g(x)\sum_{j=1}^ks_{p_\mu,j}R_j((p_\mu)_{<j},(p_\mu)_j,\mu)
\le
\sum_{j=1}^k|R_j((p_\mu)_{<j},(p_\mu)_j,\mu)|
\le
k\Delta_Q.
\]
Taking the expectation under \(w_t\) preserves the bound.  Since the argument holds for every \(\mu\),
\[
\max_{\mu\in\mathcal M}\min_{\pi\in\Delta(\mathcal P_Q)}F(\pi,\mu)
\le k\Delta_Q.
\]
Since \(F\) is affine in \(\pi\), Assumption~\ref{ass:upper-analytic}(i) allows us to apply Sion's minimax theorem~\citep{sion1958general}.  It identifies this value with the value in which the prediction is chosen first:
\[
\min_\pi\max_\mu F(\pi,\mu)=\max_\mu\min_\pi F(\pi,\mu)\le k\Delta_Q.
\]
\end{proof}

The exponential weights lemma bounds the largest total gain of any expert by the sum of the expected gains under \(w_t\), together with a regret term.  Lemma~\ref{lem:upper-minimax-value} bounds each round's expected gain after conditioning on the history and the current context.  We now combine the two bounds.

\begin{proof}[Proof of Theorem~\ref{thm:upper-online}]
Lemma~\ref{lem:upper-exponential-weights} gives
\[
\E\left[\max_{(g,s)\in\mathcal G\times\Sigma}\sum_{t=1}^T h_t(g,s)\right]
\le
\E\left[\sum_{t=1}^T\sum_{g\in\mathcal G}\sum_{s\in\Sigma}w_t(g,s)h_t(g,s)\right]
+
C_k R_{\max}\sqrt{T\bigl(\log|\mathcal G|+k|\mathcal P_Q|\bigr)}.
\]
Condition on \(\mathcal F_{t-1}\) and \(X_t\), and let \(\mu_t\) be the conditional law of \(Y_t\).  By the hypothesis of Theorem~\ref{thm:upper-online}, \(\mu_t\in\mathcal M\).  The conditional expected gain under \(w_t\) is the objective in \eqref{eq:upper-minimax} evaluated at \(\pi_t(\cdot\mid X_t)\) and \(\mu_t\).  It is therefore at most the worst-case value for \(\pi_t(\cdot\mid X_t)\).  The choice of \(\pi_t\) bounds this worst-case value by the minimax value plus \(\rho\), and Lemma~\ref{lem:upper-minimax-value} bounds the minimax value by \(k\Delta_Q\).  Hence
\[
\E\left[
\sum_{g\in\mathcal G}\sum_{s\in\Sigma}w_t(g,s)h_t(g,s)
\middle|\mathcal F_{t-1},X_t
\right]
\le
k\Delta_Q+\rho.
\]
Summing over \(t\) and taking expectations bounds the cumulative expected gain by \(T(k\Delta_Q+\rho)\).  Dividing the exponential weights bound by \(T\) and using the representation at the beginning of this subsection proves the first bound in Theorem~\ref{thm:upper-online}.  The second follows from the two bounds in Assumption~\ref{ass:upper-analytic}(iii).
\end{proof}

\subsection{The Averaged Batch Predictor and Sample Complexity}

Let \(\mathsf D\) be an \(\mathcal M\)-compatible distribution, and let \(S=((X_1,Y_1),\ldots,(X_T,Y_T))\sim\mathsf D^T\).  We run Algorithm~\ref{alg:upper-online} on the sample sequence and return the averaged predictor \(\Pi_S:\mathcal X\to\Delta(\mathcal P_Q)\) defined by
\begin{equation}\label{eq:upper-averaged-predictor}
\Pi_S(p\mid x)
:=
\frac1T\sum_{t=1}^T \pi_t(p\mid x).
\end{equation}

\subsubsection{Online-to-Batch Reduction}

The population multicalibration error of the averaged predictor need not equal the empirical multicalibration error of the online transcript.  After expressing both errors as maxima over groups and sign arrays, we compare the corresponding population and empirical objectives.  As in \citet{collina2026sample}, for each fixed group and sign array, the difference between these objectives is a martingale sum.  The following maximal inequality controls all these sums simultaneously.

\begin{lemma}[A Maximal Inequality for Finitely Many Martingales]\label{lem:upper-martingale-max}
Let \(\mathcal A\) be a finite set of size \(N\).  For each \(e\in\mathcal A\), let \((M_t(e))_{t=1}^T\) be martingale differences with respect to the same filtration, and assume \(|M_t(e)|\le G\) almost surely for all \(t,e\).  Then
\[
\E\max_{e\in\mathcal A}\sum_{t=1}^T M_t(e)
\le
G\sqrt{2T\log N}.
\]
\end{lemma}

\begin{proof}
For any \(\eta>0\), conditional Hoeffding gives
\[
\E\left[\exp\left(\eta\sum_{t=1}^T M_t(e)\right)\right]
\le
\exp\left(\frac{\eta^2G^2T}{2}\right).
\]
Therefore
\[
\E\max_e\sum_tM_t(e)
\le
\frac1\eta\log\E\sum_e\exp\left(\eta\sum_tM_t(e)\right)
\le
\frac{\log N}{\eta}+\frac{\eta G^2T}{2}.
\]
Optimizing over \(\eta\) proves the bound.
\end{proof}

We now apply Lemma~\ref{lem:upper-martingale-max} to compare the population multicalibration error of the averaged predictor with the empirical multicalibration error of the online transcript.

\begin{lemma}[Online-to-Batch Transfer]\label{lem:upper-online-to-batch}
Under Assumption~\ref{ass:upper-analytic}, for the averaged predictor \(\Pi_S\) in \eqref{eq:upper-averaged-predictor},
\[
\E
\left[
\MCErr^\Gamma_{\mathsf D}(\Pi_S;\mathcal G)
\right]
\le
\E\widehat{\MCErr}^{\Gamma}_T
+
C_k R_{\max}\sqrt{\frac{Q^k+\log|\mathcal G|}{T}},
\]
where \(C_k\) depends only on \(k\).  Both expectations are over the sample and any internal randomization used by the forecaster.
\end{lemma}

\begin{proof}
For \((g,s)\in\mathcal G\times\Sigma\), set
\[
\mathcal L_{g,s}(\Pi_S)
:=
\sum_{p\in\mathcal P_Q}\sum_{j=1}^k
s_{p,j}\nu^{\mathsf D,\Pi_S}_{g,j}(\{p\}).
\]
Maximizing over each \(s_{p,j}\in\{\pm1\}\) gives
\[
\MCErr^\Gamma_{\mathsf D}(\Pi_S;\mathcal G)
=
\max_{g\in\mathcal G}\max_{s\in\Sigma}\mathcal L_{g,s}(\Pi_S).
\]
Its empirical counterpart is
\[
\widehat{\mathcal L}_{g,s}(S)
:=
\frac1T\sum_{t=1}^T
 g(X_t)
	 \sum_{p\in\mathcal P_Q}\pi_t(p\mid X_t)
	 \sum_{j=1}^ks_{p,j}R_j(p_{<j},p_j,Y_t).
\]
Then
\[
\widehat{\MCErr}^{\Gamma}_T
=
\max_{g,s}\widehat{\mathcal L}_{g,s}(S).
\]
Therefore
\[
\MCErr^\Gamma_{\mathsf D}(\Pi_S;\mathcal G)
\le
\widehat{\MCErr}^{\Gamma}_T
+
\max_{g,s}\left(\mathcal L_{g,s}(\Pi_S)-\widehat{\mathcal L}_{g,s}(S)\right).
\]
Fix \((g,s)\).  For each \(t\), let
\[
Z_t^{g,s}
:=
	 g(X_t)
	 \sum_{p\in\mathcal P_Q}\pi_t(p\mid X_t)
	 \sum_{j=1}^ks_{p,j}R_j(p_{<j},p_j,Y_t).
\]
Define
\[
\bar Z_t^{g,s}
:=
\E_{(X,Y)\sim\mathsf D}
\left[
	 g(X)
	 \sum_{p\in\mathcal P_Q}\pi_t(p\mid X)
	 \sum_{j=1}^ks_{p,j}R_j(p_{<j},p_j,Y)
\middle| \mathcal F_{t-1}
\right].
\]
Conditionally on \(\mathcal F_{t-1}\), the rule \(\pi_t\) is fixed.  The linearity of \(\mathcal L_{g,s}\) in the predictor and the definition of \(\Pi_S\) therefore give
\[
\mathcal L_{g,s}(\Pi_S)
=
\frac1T\sum_{t=1}^T\bar Z_t^{g,s}.
\]
The empirical quantity satisfies
\[
\widehat{\mathcal L}_{g,s}(S)
=
\frac1T\sum_{t=1}^T Z_t^{g,s}.
\]
Because the sample is i.i.d., \(\bar Z_t^{g,s}\) is the conditional expectation of \(Z_t^{g,s}\) given \(\mathcal F_{t-1}\).  Hence the differences
\[
M_t^{g,s}:=\bar Z_t^{g,s}-Z_t^{g,s}
\]
are martingale differences with respect to the filtration generated by the transcript and the forecaster's internal randomization.  Since group functions take values in \([0,1]\), we have \(|Z_t^{g,s}|\le kR_{\max}\), and hence \(|M_t^{g,s}|\le2kR_{\max}\).  Subtracting the two equalities gives
\[
\mathcal L_{g,s}(\Pi_S)-\widehat{\mathcal L}_{g,s}(S)
=
\frac1T\sum_{t=1}^T M_t^{g,s}.
\]
Lemma~\ref{lem:upper-martingale-max}, applied to the \(|\mathcal G||\Sigma|=|\mathcal G|2^{k|\mathcal P_Q|}\) martingale-difference sequences, gives
\[
\E\max_{g,s}\sum_{t=1}^T M_t^{g,s}
\le
C_k R_{\max}\sqrt{T\bigl(\log|\mathcal G|+k|\mathcal P_Q|\bigr)}.
\]
Dividing by \(T\) proves the lemma.
\end{proof}

\subsubsection{Sample Complexity}

Combining the empirical guarantee for the online forecaster with the online-to-batch transfer gives the sample-complexity bound.

\begin{theorem}[Sample Complexity under Approximate Minimax Optimization]\label{thm:upper-batch}
Suppose Assumption~\ref{ass:upper-analytic} holds.  Let \(\mathsf D\) be any \(\mathcal M\)-compatible distribution on \(\mathcal X\times\mathcal Y\), and let \(\mathcal G\) be any finite group family.  Fix a sufficiently small target error \(\eps>0\) and a minimax accuracy
\[
0\le\rho<\frac\eps3.
\]
There is a randomized learner using minimax accuracy \(\rho\) and a grid with
\[
Q=\Theta\!\left((\eps-3\rho)^{-1}\right)
\]
which, from
\[
n
\ge
C
\left(
(\eps-3\rho)^{-(k+2)}
+
(\eps-3\rho)^{-2}\log|\mathcal G|
\right)
\]
i.i.d. samples from \(\mathsf D\), outputs a finite-support predictor \(\Pi_S:\mathcal X\to\Delta(\mathcal P_Q)\) satisfying
\[
\Prob\left[
\MCErr^\Gamma_{\mathsf D}(\Pi_S;\mathcal G)\le\eps
\right]
\ge
\frac23.
\]
The probability is over the sample and any internal randomization of the learner, and \(C\) depends only on \(k\), \(R_{\max}\), and \(C_{\mathrm{grid}}\).
\end{theorem}

\begin{proof}[Proof of Theorem~\ref{thm:upper-batch}]
Run the online forecaster for \(T=n\) rounds.  Because the sample is i.i.d., the conditional law of \(Y_t\) given \(\mathcal F_{t-1}\) and \(X_t\) is \(\mathsf D_{Y\mid X=X_t}\).  The \(\mathcal M\)-compatibility of \(\mathsf D\) therefore ensures that the hypothesis of Theorem~\ref{thm:upper-online} holds.  Theorem~\ref{thm:upper-online} and Lemma~\ref{lem:upper-online-to-batch} now give
\[
\E\MCErr^\Gamma_{\mathsf D}(\Pi_S;\mathcal G)
\le
\rho
+
C'
\left(
\frac1Q
+
R_{\max}\sqrt{\frac{Q^k+\log|\mathcal G|}{n}}
\right).
\]
Because \(\eps-3\rho>0\), take
\[
Q
=
\left\lceil
\frac{6C'}{\eps-3\rho}
\right\rceil.
\]
If
\[
n
\ge
C
\left(
(\eps-3\rho)^{-(k+2)}
+
(\eps-3\rho)^{-2}\log|\mathcal G|
\right)
\]
for a sufficiently large constant \(C\), then the expectation is at most
\[
\rho+\frac{\eps-3\rho}{3}
=
\frac\eps3.
\]
Markov's inequality therefore gives
\[
\Prob\left[\MCErr^\Gamma_{\mathsf D}(\Pi_S;\mathcal G)>\eps\right]
\le
\frac13.
\]
\end{proof}

If \(\rho\le\eps/6\), then \(\eps-3\rho\ge\eps/2\), so the theorem immediately gives the following bound.

\begin{corollary}[Upper Bound on Sample Complexity]\label{cor:upper-sample}
Under the conditions of Theorem~\ref{thm:upper-batch}, suppose each inner minimax problem is solved to additive accuracy
\[
\rho\le\frac\eps6.
\]
Then there is a randomized learner using \(Q=\Theta(1/\eps)\) and
\[
n
\ge
C
\left(
\eps^{-(k+2)}+\eps^{-2}\log|\mathcal G|
\right)
\]
i.i.d. samples from \(\mathsf D\) whose output satisfies
\[
\Prob\left[
\MCErr^\Gamma_{\mathsf D}(\Pi_S;\mathcal G)\le\eps
\right]
\ge
\frac23.
\]
In particular, if \(|\mathcal G|\le\lceil\eps^{-\kappa}\rceil\) for a fixed \(\kappa>0\), then
\[
n=O(\eps^{-(k+2)}).
\]
\end{corollary}

Consequently,
\[
\SC^{(\kappa)}_{\Gamma,\mathcal M}(\eps)
=
O(\eps^{-(k+2)}).
\]
Binary groups are included among the group functions allowed in the minimax definition.  Thus, if Assumption~\ref{ass:witness} also holds, Theorem~\ref{thm:lower-bound}
and Corollary~\ref{cor:upper-sample} give
\[
\SC^{(\kappa)}_{\Gamma,\mathcal M}(\eps)
=
\widetilde{\Theta}\!\left(\eps^{-(k+2)}\right).
\]

\subsection{Polynomial-Time Implementation}\label{sec:upper-polytime}

The exponential weights procedure analyzed above is information-theoretic and maintains weights over \(|\mathcal G|2^{k|\mathcal P_Q|}\) experts indexed by sign patterns.  This section gives a polynomial-time implementation under an explicit linear optimization oracle over \(\mathcal M\).

All running-time statements in this paper use a unit-cost real-arithmetic model.  Arithmetic operations, comparisons, evaluations of the elementary functions used in the implementation, and evaluations of the supplied group and residual functions are exact and have unit cost.  The running time also includes the operations performed by the linear optimization oracle, as specified in Assumption~\ref{ass:upper-oracle}.

There are two computational tasks.  We first compute the exponential weights distribution without enumerating all sign arrays.  We then solve the minimax problem using linear optimization over \(\mathcal M\).

\subsubsection{Computing Exponential Weights Implicitly}

Let
\[
\CumRes^{t-1}_{g,p,j}
:=
\sum_{\tau<t}
 g(X_\tau)\pi_\tau(p\mid X_\tau)R_j(p_{<j},p_j,Y_\tau).
\]
The definition of the gain gives
\[
C_{t-1}(g,s)
=
\sum_{\tau<t}h_\tau(g,s)
=
\sum_{p\in\mathcal P_Q}\sum_{j=1}^k
s_{p,j}\CumRes^{t-1}_{g,p,j}.
\]
Hence the exponential weights distribution over \((g,s)\in\mathcal G\times\Sigma\), with learning rate \(\eta_{\mathrm{EW}}>0\), has weight proportional to
\[
\exp\left(
\eta_{\mathrm{EW}}
\sum_{p\in\mathcal P_Q}\sum_{j=1}^k
s_{p,j}\CumRes^{t-1}_{g,p,j}
\right).
\]
For fixed \(g\), the exponent is a sum of terms that each depend on only one sign coordinate, so the conditional distribution of the signs factorizes.  Using
\[
\sum_{\sigma\in\{\pm1\}}e^{\eta\sigma a}
=
2\cosh(\eta a),
\]
the group marginal is
\[
\omega_t(g)
=
\frac{
\prod_{p,j}2\cosh\bigl(\eta_{\mathrm{EW}}\CumRes^{t-1}_{g,p,j}\bigr)
}{
\sum_{g'\in\mathcal G}
\prod_{p,j}2\cosh\bigl(\eta_{\mathrm{EW}}\CumRes^{t-1}_{g',p,j}\bigr)
},
\]
and, conditional on \(g\), each sign has mean
\[
m_{t,g,p,j}
:=
\E[s_{p,j}\mid g]
=
\tanh\bigl(\eta_{\mathrm{EW}}\CumRes^{t-1}_{g,p,j}\bigr).
\]
Conditional on \(g\), the objective in \eqref{eq:upper-minimax} is linear in the signs, so its average over the signs depends only on the means \(m_{t,g,p,j}\).  Averaging also over the group marginal \(\omega_t\), the objective becomes
\[
\sum_{p\in\mathcal P_Q}\pi_p
\sum_{j=1}^k
 a_{t,p,j}(x)R_j(p_{<j},p_j,\mu),
\]
where
\[
a_{t,p,j}(x)
:=
\sum_{g\in\mathcal G}
\omega_t(g)g(x)m_{t,g,p,j}.
\]
Thus it suffices to maintain the \(O(k|\mathcal G||\mathcal P_Q|)\) values \(\CumRes^{t-1}_{g,p,j}\), rather than enumerate \(2^{k|\mathcal P_Q|}\) sign patterns.

\subsubsection{The Optimization Oracle}

Once the coefficients \(a_{t,p,j}(x)\) have been computed, it remains to solve the minimax problem.  \citet{noarov2025high} solve an analogous problem for high-dimensional unbiased prediction by expressing it as a finite-dimensional linear program with continuously many constraints and using an approximate separation oracle.  Here the required approximate separation step reduces to linear optimization over \(\mathcal M\).

\begin{assumption}[Linear Optimization Oracle over \(\mathcal M\)]\label{ass:upper-oracle}
There is an oracle \(\mathsf{OPT}_{\mathcal M}\) which, for every requested accuracy \(\xi\in(0,1)\) and given coefficients \(c_{p,j}\in\R\), returns a law \(\widehat\mu\in\mathcal M\) satisfying
\[
\sum_{p\in\mathcal P_Q}\sum_{j=1}^k
c_{p,j}R_j(p_{<j},p_j,\widehat\mu)
\ge
\sup_{\mu\in\mathcal M}
\sum_{p\in\mathcal P_Q}\sum_{j=1}^k
c_{p,j}R_j(p_{<j},p_j,\mu)
\;-
\xi
\]
using a number of operations polynomial in \(|\mathcal P_Q|\) and \(\log(1/\xi)\).  Within the same bound, the objective value and the residual expectations at \(\widehat\mu\) can be evaluated to additive accuracy \(\xi\).

For the computational implementation, fix one deterministic,
Borel-measurable realization of this oracle, including deterministic
tie-breaking whenever more than one admissible output is available.
\end{assumption}

Given the coefficients \(a_{t,p,j}(x)\), the inner minimax problem is the semi-infinite linear program
\[
\begin{aligned}
\min_{\pi,\zeta}\quad &\zeta\\
\text{s.t.}\quad &\pi\in\Delta(\mathcal P_Q),\\
&\sum_{p,j}\pi_p a_{t,p,j}(x)R_j(p_{<j},p_j,\mu)
\le \zeta
\qquad \forall \mu\in\mathcal M.
\end{aligned}
\]
For a candidate \((\pi,\zeta)\), a violated constraint is found by calling \(\mathsf{OPT}_{\mathcal M}\) with coefficients
\[
c_{p,j}=\pi_p a_{t,p,j}(x).
\]
To see how the two optimization accuracies enter, call the oracle with internal accuracy \(\xi\), let \(\widehat\mu\) be the returned law, and evaluate its objective to the same accuracy.  If the resulting estimate exceeds \(\zeta+\xi\), then the constraint associated with \(\widehat\mu\) is violated.  Otherwise, the oracle and evaluation guarantees certify that every constraint holds with additive slack at most \(3\xi\).  Since \(|a_{t,p,j}(x)|\le1\) and the residuals are bounded, \(\zeta\) may be restricted to a bounded interval.  Thus the linear program has a bounded finite-dimensional feasible region, and the oracle supplies approximate separation for its continuously many constraints.  Given a target minimax accuracy \(\rho\in(0,1)\), the weak ellipsoid method~\citep{grotschel1988geometric} chooses its internal tolerances, including \(\xi\), with \(\log(1/\xi)\) polynomial in \(|\mathcal P_Q|\) and \(\log(1/\rho)\).  It returns a distribution \(\pi_t(\cdot\mid x)\) satisfying \eqref{eq:upper-minimax} using a number of operations and oracle calls polynomial in the same quantities.

Fix a deterministic implementation of the weak ellipsoid method, including
deterministic tie-breaking, and combine it with the fixed oracle from
Assumption~\ref{ass:upper-oracle}.  Denote the resulting solver by
\[
\mathsf{Solve}_\rho(H,x)\in\Delta(\mathcal P_Q),
\]
where \(H=(\CumRes_{g,p,j})_{g,p,j}\) is the array of cumulative residuals
available before the round.  The internal tolerances are chosen so that
\(\mathsf{Solve}_\rho(H,x)\) satisfies \eqref{eq:upper-minimax} with
additive error at most \(\rho\).  In the computational implementation, the
rule selected at round \(t\) is defined for every context by
\[
\pi_t(\cdot\mid x)
:=
\mathsf{Solve}_\rho\!\left(
\bigl(\CumRes^{t-1}_{g,p,j}\bigr)_{g,p,j},
x
\right).
\]
Thus the rule is fixed by the pre-round state and the context, and the
same rule can be reproduced after training.

We can therefore compute the exponential weights distribution and solve the minimax problem in polynomial time; the following theorem records the resulting runtime guarantee.

\begin{theorem}[Polynomial-Time Implementation]\label{thm:upper-polytime}
Assume Assumptions~\ref{ass:upper-analytic} and~\ref{ass:upper-oracle}.  Then, for every minimax accuracy \(\rho\in(0,1)\), the learner can be implemented with training time polynomial in
\[
n,
\quad Q^k,
\quad |\mathcal G|,
\quad \log(1/\rho).
\]
Choosing \(\rho=\eps/6\) and the remaining parameters as in Corollary~\ref{cor:upper-sample} implements its learner.  In particular, for every fixed \(\kappa>0\), if \(|\mathcal G|\le\eps^{-\kappa}\), then the learner runs in time polynomial in \(1/\eps\).

The output predictor has a representation of polynomial size, and for any
queried context \(x\), either its full prediction distribution or an exact
sample from that distribution can be computed in time polynomial in the
same parameters.
\end{theorem}

\begin{proof}
The formulas for \(\omega_t(g)\) and \(m_{t,g,p,j}\) represent the exact exponential weights distribution over the sign-pattern experts without enumerating them, while the formula for \(a_{t,p,j}(x)\) computes the resulting coefficients in the minimax objective.  These calculations require storing only the \(O(k|\mathcal G||\mathcal P_Q|)\) values \(\CumRes^{t-1}_{g,p,j}\).  At each round, the coefficients \(a_{t,p,j}(x)\) can be computed in time polynomial in \(k|\mathcal G||\mathcal P_Q|\) for any queried context \(x\).  During training, they need only be computed at the observed context \(X_t\).  The inner minimax problem has \(|\mathcal P_Q|+1\) variables, and the approximate-separation argument in the preceding subsection computes a \(\rho\)-approximate solution in the stated running time.  The \(\rho\)-approximate minimax solution contributes the additive \(\rho\) term in \eqref{eq:upper-empirical-bound}.  Finally, choose \(\rho=\eps/6\), \(Q=\Theta(1/\eps)\), and \(n=O(\eps^{-(k+2)}+\eps^{-2}\log|\mathcal G|)\) as in Corollary~\ref{cor:upper-sample}.  These choices give the claimed polynomial runtime in \(1/\eps\).

To represent the output, store the fixed solver specification together
with the \(T\) pre-round arrays
\[
H_{t-1}
:=
\bigl(\CumRes^{t-1}_{g,p,j}\bigr)_{g,p,j},
\qquad t\in[T].
\]
This uses \(O(Tk|\mathcal G||\mathcal P_Q|)\) real numbers.  For a queried
context \(x\), rerunning the same deterministic map
\(\mathsf{Solve}_\rho(H_{t-1},x)\) reproduces the rule
\(\pi_t(\cdot\mid x)\) used by the learner.  The full averaged distribution
\[
\Pi_S(\cdot\mid x)
=
\frac1T\sum_{t=1}^T
\mathsf{Solve}_\rho(H_{t-1},x)
\]
is obtained by \(T\) solver calls.  Alternatively, an exact sample from
\(\Pi_S(\cdot\mid x)\) is obtained by drawing
\(t\sim\Unif\{1,\ldots,T\}\), computing
\(\mathsf{Solve}_\rho(H_{t-1},x)\), and sampling from that distribution.
Both the representation size and these evaluation procedures are
polynomial in \(n\), \(Q^k\), \(|\mathcal G|\), and
\(\log(1/\rho)\).
\end{proof}

\section{Canonical Instantiations}\label{sec:instantiations}

We now show that the general lower and upper bounds apply to three multilevel properties: mean and mean absolute deviation; mean, variance, and skewness; and quantile and CVaR.  In each case, they give matching sample-complexity bounds up to logarithmic factors.  The lower bound uses a local family of outcome distributions, while the upper bound may hold over a larger class and uses an explicit optimization oracle.

\subsection{Mean, Mean Absolute Deviation}

Let \(\mathcal M\) be the class of all probability laws on \([0,1]\), and use the prediction space
\[
\mathcal P=[0,1]^2.
\]
For \(\mu\in\mathcal M\), define
\[
m(\mu):=\E_\mu Y,
\qquad
d(\mu):=\E_\mu |Y-m(\mu)|.
\]
The two-level property is
\[
\Gamma(\mu)=(m(\mu),d(\mu)).
\]
It is sequentially conditionally identifiable with residual functions
\[
R_1(\varnothing,m,y)=m-y,
\qquad
R_2(m,d,y)=d-|y-m|.
\]
The second coordinate is the expected \(\ell_1\) loss evaluated at the first coordinate.  It is therefore analogous to variance, but uses absolute loss rather than squared loss.  It is not a bilevel Bayes pair, since absolute loss elicits medians rather than means.

We first construct a local witness family on the three-point support
\[
\mathcal Y_{\mathrm{MAD}}=\{0,1/2,1\}.
\]
Let \(v^\circ=(7/20,7/20)\), and let \(\mathcal R_{\mathrm{MAD}}\) be a sufficiently small closed rectangle around \(v^\circ\), contained in \((0,1/2)\times(0,1)\).  For \(v=(v_1,v_2)\in\mathcal R_{\mathrm{MAD}}\), define \(\mu_v\) by
\[
\Prob_{Y\sim\mu_v}(Y=0)=\frac{v_2}{2v_1},
\]
\[
\Prob_{Y\sim\mu_v}(Y=1/2)=2(1-v_1)-\frac{v_2}{v_1},
\]
\[
\Prob_{Y\sim\mu_v}(Y=1)=2v_1-1+\frac{v_2}{2v_1}.
\]
At \(v^\circ\), these probabilities are \(1/2,3/10,1/5\).  Since they depend continuously on \(v\), \(\mathcal R_{\mathrm{MAD}}\) can be chosen small enough that each is uniformly bounded below by a positive constant throughout the rectangle.  Since \(v_1<1/2\) on \(\mathcal R_{\mathrm{MAD}}\), direct calculation gives
\[
\E_{\mu_v}Y=v_1,
\qquad
\E_{\mu_v}|Y-v_1|=v_2.
\]

\begin{proposition}[Mean, Mean Absolute Deviation: Witness]\label{prop:mad-witness}
The family \(\{\mu_v:v\in\mathcal R_{\mathrm{MAD}}\}\) satisfies Assumption~\ref{ass:witness}.
\end{proposition}

\begin{proof}
We verified above that \(\Gamma(\mu_v)=v\), as required by Assumption~\ref{ass:witness}(i).  The residuals evaluated at the true prefix are
\[
R_1(v_{<1},q,\mu_v)=q-v_1,
\qquad
R_2(v_1,q,\mu_v)=q-v_2.
\]
Therefore conditions~(ii) and~(iii) in Assumption~\ref{ass:witness} hold.

It remains to verify Assumption~\ref{ass:witness}(iv).  The first coordinate has no preceding predictions.  For the second coordinate,
\[
\begin{aligned}
&\left|
R_2(p_1,p_2,\mu_v)-R_2(v_1,p_2,\mu_v)
\right|
\\
&\qquad=
\left|
\E_{\mu_v}\bigl[|Y-v_1|-|Y-p_1|\bigr]
\right|
\le
|p_1-v_1|,
\end{aligned}
\]
where the last step uses the reverse triangle inequality.  This verifies Assumption~\ref{ass:witness}(iv).

Finally, the probability vector of \(\mu_v\) is a smooth function of \(v=(v_1,v_2)\) on \(\mathcal R_{\mathrm{MAD}}\), and all three coordinates are uniformly bounded below.  Hence, for some constant \(C<\infty\),
\[
\KL(\mu_v\,\|\,\mu_{v'})
\le
C\|v-v'\|_2^2.
\]
This verifies Assumption~\ref{ass:witness}(v).
\end{proof}

\begin{corollary}[Mean, Mean Absolute Deviation: Lower Bound]
Fix \(\kappa>0\).  There exist constants \(c,C,\eps_0>0\), depending only on \(\kappa\), such that the following holds for every \(0<\eps\le\eps_0\).  For the property
\[
\Gamma=(\E Y,\E|Y-\E Y|),
\]
one can construct a finite context space, a binary group family of size at most \(\eps^{-\kappa}\), and a finite collection of data distributions whose conditional outcome laws are supported on three points in \([0,1]\), such that every learner achieving \(\Gamma\)-ECE at most \(\eps\) with probability at least \(2/3\) must use
\[
n
\ge
c\frac{\eps^{-4}}{\log^4(C/\eps)}.
\]
Equivalently,
\[
\SC^{(\kappa)}_{\Gamma,\mathcal M}(\eps)
=
\widetilde{\Omega}(\eps^{-4}).
\]
\end{corollary}

\begin{proof}
Apply Theorem~\ref{thm:lower-bound} using Proposition~\ref{prop:mad-witness}.
\end{proof}

We next verify the upper-bound conditions over \(\mathcal M\) and give an exact optimization oracle.

\begin{proposition}[Mean, Mean Absolute Deviation: Conditions for the Upper Bound]\label{prop:upper-mad}
The regularity conditions in Assumption~\ref{ass:upper-analytic} hold over \(\mathcal M\).
Moreover, the linear optimization oracle in Assumption~\ref{ass:upper-oracle} has an exact polynomial-time implementation.
\end{proposition}

\begin{proof}
The law class \(\mathcal M\) is convex and compact in the weak topology.  For each fixed prediction \((m,d)\), the maps
\[
\mu\mapsto \E_\mu[m-Y],
\qquad
\mu\mapsto \E_\mu\!\left[d-|Y-m|\right]
\]
are affine and continuous.  These observations verify Assumption~\ref{ass:upper-analytic}(i).  Condition~(ii) also holds because both coordinates of the prediction and the outcome lie in \([0,1]\).

For two predictions \((m,d),(m',d')\),
\[
\left|R_1(\varnothing,m,\mu)-R_1(\varnothing,m',\mu)\right|
=
|m-m'|,
\]
and
\[
\left|R_2(m,d,\mu)-R_2(m',d',\mu)\right|
\le
|d-d'|+|m-m'|.
\]
Thus \(R(\cdot,\mu)\) is uniformly Lipschitz.  A rectangular grid with mesh \(O(1/Q)\) therefore satisfies the required approximation bound and has \(O(Q^2)\) points, verifying Assumption~\ref{ass:upper-analytic}(iii).

For the oracle, given coefficients \(c_{p,j}\), write each grid point as \(p=(m_p,d_p)\), and define
\[
\phi(y)
:=
\sum_{p\in\mathcal P_Q}
\left[
c_{p,1}(m_p-y)
+c_{p,2}\bigl(d_p-|y-m_p|\bigr)
\right].
\]
Maximizing \(\E_\mu\phi(Y)\) over all probability laws on \([0,1]\) is the same as maximizing \(\phi(y)\) over \(y\in[0,1]\), because an optimizer may be taken to be a point mass at a maximizer.  The function \(\phi\) is piecewise affine with breakpoints among the grid values \(m_p\), so a maximizer lies in \(\{0,1\}\cup\{m_p:p\in\mathcal P_Q\}\).  Checking these points gives an exact polynomial-time oracle.
\end{proof}

\begin{corollary}[Mean, Mean Absolute Deviation: Upper Bound]\label{cor:upper-mad}
Fix \(\kappa>0\).  There exist constants \(C,\eps_0>0\), depending only on \(\kappa\), such that the following holds for every \(0<\eps\le\eps_0\).  Let
\[
\Gamma=(\E Y,\E|Y-\E Y|)
\]
be the property on \([0,1]\).  For every data distribution \(\mathsf D\) on \(\mathcal X\times[0,1]\) and every finite group family satisfying \(|\mathcal G|\le\lceil\eps^{-\kappa}\rceil\), there is a randomized learner using at most \(C\eps^{-4}\) samples whose output \(\Pi_S\) satisfies
\[
\Prob\left[
\MCErr^\Gamma_{\mathsf D}(\Pi_S;\mathcal G)\le\eps
\right]
\ge
\frac23.
\]
The exact oracle in Proposition~\ref{prop:upper-mad} gives training time polynomial in \(1/\eps\).
\end{corollary}

\begin{proof}
Apply Corollary~\ref{cor:upper-sample} and Theorem~\ref{thm:upper-polytime} using Proposition~\ref{prop:upper-mad}.
\end{proof}

\subsection{Mean, Variance, Skewness}

Let
\[
\mathcal Y_{\mathrm{MVS}}=\{0,1/3,2/3,1\}.
\]
Fix \(h\in(0,1/4)\), and let \(\mathcal M_h\) be the set of probability laws on \(\mathcal Y_{\mathrm{MVS}}\) that assign mass at least \(h\) to every atom.  This class is convex and compact.  Moreover, every \(\mu\in\mathcal M_h\) has variance bounded below:
\[
\Var_\mu(Y)
=
\frac12\sum_{y,y'\in\mathcal Y_{\mathrm{MVS}}}
\mu(y)\mu(y')(y-y')^2
\ge
\mu(0)\mu(1)
\ge
h^2.
\]
For \(\mu\in\mathcal M_h\), define
\[
m(\mu):=\E_\mu Y,
\qquad
\sigma^2(\mu):=\E_\mu(Y-m(\mu))^2,
\]
and the standardized skewness
\[
\gamma(\mu)
:=
\frac{\E_\mu(Y-m(\mu))^3}{\bigl(\sigma^2(\mu)\bigr)^{3/2}}.
\]
The three-level property is
\[
\Gamma(\mu)=(m(\mu),\sigma^2(\mu),\gamma(\mu)).
\]
Since \(|Y-m(\mu)|\le1\) and \(\sigma^2(\mu)\ge h^2\), its range lies in
\[
\mathcal P
:=
[0,1]\times[h^2,1/4]\times[-h^{-3},h^{-3}].
\]
It is sequentially conditionally identifiable with residual functions
\[
R_1(\varnothing,m,y)=m-y,
\]
\[
R_2(m,\sigma^2,y)=\sigma^2-(y-m)^2,
\]
\[
R_3((m,\sigma^2),\gamma,y)=\gamma\,(\sigma^2)^{3/2}-(y-m)^3.
\]
The third residual depends on both preceding coordinates and identifies skewness only after the mean and variance have been fixed.

We now construct a local three-dimensional witness family.  Let \(\mu^\circ\) be the uniform distribution on \(\mathcal Y_{\mathrm{MVS}}\), and let
\[
v^\circ:=\Gamma(\mu^\circ)=\left(\frac12,\frac5{36},0\right).
\]
For a target vector \(v=(v_1,v_2,v_3)\) with \(v_2>0\), define the corresponding raw moments
\[
r_1(v):=v_1,
\qquad
r_2(v):=v_1^2+v_2,
\qquad
r_3(v):=v_1^3+3v_1v_2+v_3v_2^{3/2}.
\]
Define weights on \(\mathcal Y_{\mathrm{MVS}}\) by
\[
\begin{aligned}
\mu_v(0)
&:=
1-\frac{11}{2}r_1(v)+9r_2(v)-\frac92r_3(v),
\\
\mu_v(1/3)
&:=
9r_1(v)-\frac{45}{2}r_2(v)+\frac{27}{2}r_3(v),
\\
\mu_v(2/3)
&:=
-\frac92r_1(v)+18r_2(v)-\frac{27}{2}r_3(v),
\\
\mu_v(1)
&:=
r_1(v)-\frac92r_2(v)+\frac92r_3(v).
\end{aligned}
\]
A direct calculation gives
\[
\sum_y\mu_v(y)=1,
\qquad
\sum_y\mu_v(y)y^j=r_j(v)
\quad
\text{for }j\in\{1,2,3\}.
\]
At \(v=v^\circ\), all four weights equal \(1/4\).  Since \(h<1/4\), continuity gives a nondegenerate closed box
\[
\mathcal R_{\mathrm{MVS}}
\subset
\operatorname{int}(\mathcal P)
\]
around \(v^\circ\), small enough that \(\mu_v(y)\ge h\) for every \(v\in\mathcal R_{\mathrm{MVS}}\) and \(y\in\mathcal Y_{\mathrm{MVS}}\).  Thus \(\mu_v\in\mathcal M_h\) throughout this box.  Expanding the centered moments gives
\[
\E_{\mu_v}(Y-v_1)^2=v_2,
\qquad
\E_{\mu_v}(Y-v_1)^3=v_3v_2^{3/2},
\]
and hence
\[
\Gamma(\mu_v)=v
\qquad
\forall v\in\mathcal R_{\mathrm{MVS}}.
\]

\begin{proposition}[Mean, Variance, Skewness: Witness]\label{prop:mvs-witness}
The family \(\{\mu_v:v\in\mathcal R_{\mathrm{MVS}}\}\) satisfies Assumption~\ref{ass:witness}.
\end{proposition}

\begin{proof}
The identity \(\Gamma(\mu_v)=v\) holds by construction, as required by Assumption~\ref{ass:witness}(i).  The residuals evaluated at the true prefix are
\[
R_1(v_{<1},q,\mu_v)=q-v_1,
\qquad
R_2(v_1,q,\mu_v)=q-v_2,
\]
and
\[
R_3((v_1,v_2),q,\mu_v)
=
v_2^{3/2}(q-v_3).
\]
Since \(\mathcal R_{\mathrm{MVS}}\) is compact and \(v_2\ge h^2\), the coefficients \(1\) and \(v_2^{3/2}\) are uniformly bounded above and away from zero.  Therefore conditions~(ii) and~(iii) in Assumption~\ref{ass:witness} hold.

It remains to verify Assumption~\ref{ass:witness}(iv).  The first coordinate has no preceding predictions.  For the second coordinate,
\[
R_2(p_1,p_2,\mu_v)-R_2(v_1,p_2,\mu_v)
=
\E_{\mu_v}\!\left[(v_1-Y)^2-(p_1-Y)^2\right].
\]
Since \(p_1,v_1,Y\in[0,1]\), the above equation gives
\[
\left|
R_2(p_1,p_2,\mu_v)-R_2(v_1,p_2,\mu_v)
\right|
\le
2|p_1-v_1|.
\]
For the third coordinate,
\[
\begin{aligned}
&\left|
R_3((p_1,p_2),p_3,\mu_v)
-
R_3((v_1,v_2),p_3,\mu_v)
\right|
\\
&\qquad\le
|p_3|\,|p_2^{3/2}-v_2^{3/2}|
+
\E_{\mu_v}\left|(Y-p_1)^3-(Y-v_1)^3\right|.
\end{aligned}
\]
On \(\mathcal P\), \(|p_3|\le h^{-3}\), the map \(u\mapsto u^{3/2}\) is Lipschitz on \([h^2,1/4]\), and \((y,a)\mapsto(y-a)^3\) is Lipschitz for \(y,a\in[0,1]\).  Therefore
\[
\left|
R_3((p_1,p_2),p_3,\mu_v)
-
R_3((v_1,v_2),p_3,\mu_v)
\right|
\le
C\bigl(|p_1-v_1|+|p_2-v_2|\bigr),
\]
with \(C\) depending only on \(h\).  This verifies Assumption~\ref{ass:witness}(iv).

Finally, the explicit formulas above show that \(v\mapsto\mu_v\) is Lipschitz on \(\mathcal R_{\mathrm{MVS}}\).  Since \(\mu_{v'}(y)\ge h\) for every atom, the inequality \(\log u\le u-1\) gives
\[
\begin{aligned}
\KL(\mu_v\,\|\,\mu_{v'})
&\le
\sum_{y\in\mathcal Y_{\mathrm{MVS}}}
\frac{\bigl(\mu_v(y)-\mu_{v'}(y)\bigr)^2}{\mu_{v'}(y)}
\\
&\le
\frac1h
\sum_{y\in\mathcal Y_{\mathrm{MVS}}}
\bigl(\mu_v(y)-\mu_{v'}(y)\bigr)^2
\\
&\le
C\|v-v'\|_2^2.
\end{aligned}
\]
This verifies Assumption~\ref{ass:witness}(v).
\end{proof}

\begin{corollary}[Mean, Variance, Skewness: Lower Bound]
Fix \(h\in(0,1/4)\) and \(\kappa>0\).  There exist constants \(c,C,\eps_0>0\), depending only on \(h\) and \(\kappa\), such that the following holds for every \(0<\eps\le\eps_0\).  For the property
\[
\Gamma=(\text{mean},\text{variance},\text{skewness}),
\]
one can construct a finite context space, a binary group family of size at most \(\eps^{-\kappa}\), and a finite collection of data distributions whose conditional outcome laws belong to \(\mathcal M_h\), such that every learner achieving \(\Gamma\)-ECE at most \(\eps\) with probability at least \(2/3\) must use
\[
n
\ge
c\frac{\eps^{-5}}{\log^5(C/\eps)}.
\]
Equivalently,
\[
\SC^{(\kappa)}_{\Gamma,\mathcal M_h}(\eps)
=
\widetilde{\Omega}(\eps^{-5}).
\]
\end{corollary}

\begin{proof}
Apply Theorem~\ref{thm:lower-bound} using Proposition~\ref{prop:mvs-witness}.
\end{proof}

We next verify the upper-bound conditions over \(\mathcal M_h\) and give an exact optimization oracle.

\begin{proposition}[Mean, Variance, Skewness: Conditions for the Upper Bound]\label{prop:upper-mvs}
The regularity conditions in Assumption~\ref{ass:upper-analytic} hold over \(\mathcal M_h\).
Moreover, the linear optimization oracle in Assumption~\ref{ass:upper-oracle} has an exact polynomial-time implementation.
\end{proposition}

\begin{proof}
The law class \(\mathcal M_h\) is a convex, compact subset of the four-dimensional probability simplex.  For each fixed prediction \(p=(m,\sigma^2,\gamma)\), the maps
\[
\mu\mapsto \E_\mu[m-Y],
\]
\[
\mu\mapsto \E_\mu\!\left[\sigma^2-(Y-m)^2\right],
\]
\[
\mu\mapsto \E_\mu\!\left[\gamma\,(\sigma^2)^{3/2}-(Y-m)^3\right]
\]
are affine and continuous.  These observations verify Assumption~\ref{ass:upper-analytic}(i).  Condition~(ii) also holds because \(\mathcal P\) and \(\mathcal Y_{\mathrm{MVS}}\) are compact.

It remains to verify Assumption~\ref{ass:upper-analytic}(iii).  For two predictions \(p=(m,\sigma^2,\gamma)\) and \(p'=(m',{\sigma'}^2,\gamma')\) in \(\mathcal P\),
\[
\left|R_1(\varnothing,m,\mu)-R_1(\varnothing,m',\mu)\right|
=
|m-m'|,
\]
\[
\left|R_2(m,\sigma^2,\mu)-R_2(m',{\sigma'}^2,\mu)\right|
\le
|\sigma^2-{\sigma'}^2|+2|m-m'|,
\]
and
\[
\left|R_3((m,\sigma^2),\gamma,\mu)-R_3((m',{\sigma'}^2),\gamma',\mu)\right|
\le
C_h\bigl(|m-m'|+|\sigma^2-{\sigma'}^2|+|\gamma-\gamma'|\bigr),
\]
because \(\gamma\) is bounded, \(u\mapsto u^{3/2}\) is Lipschitz on \([h^2,1/4]\), and \((y,m)\mapsto(y-m)^3\) is Lipschitz on \(\mathcal Y_{\mathrm{MVS}}\times[0,1]\).  Thus \(R(\cdot,\mu)\) is uniformly Lipschitz on \(\mathcal P\).  A rectangular grid with mesh \(O(1/Q)\) therefore satisfies the required approximation bound and has \(O(Q^3)\) points.

For the oracle, given coefficients \(c_{p,j}\), write each grid point as \(p=(m_p,\sigma_p^2,\gamma_p)\), and define
\[
\phi(y)
:=
\sum_{p\in\mathcal P_Q}
\left[
c_{p,1}(m_p-y)
+c_{p,2}\bigl(\sigma_p^2-(y-m_p)^2\bigr)
+c_{p,3}\bigl(\gamma_p\,(\sigma_p^2)^{3/2}-(y-m_p)^3\bigr)
\right].
\]
The oracle problem is
\[
\sup_{\mu\in\mathcal M_h}
\sum_{y\in\mathcal Y_{\mathrm{MVS}}}\mu(y)\phi(y).
\]
This is a linear optimization problem over the simplex with lower bounds \(\mu(y)\ge h\).  An optimizer assigns mass \(h\) to each atom and the remaining mass \(1-4h\) to an atom maximizing \(\phi(y)\).  Thus the oracle is exact and polynomial time.
\end{proof}

\begin{corollary}[Mean, Variance, Skewness: Upper Bound]\label{cor:upper-mvs}
Fix \(h\in(0,1/4)\) and \(\kappa>0\).  There exist constants \(C,\eps_0>0\), depending only on \(h\) and \(\kappa\), such that the following holds for every \(0<\eps\le\eps_0\).  Let
\[
\Gamma=(\text{mean},\text{variance},\text{skewness})
\]
be the property over \(\mathcal M_h\).  For every \(\mathcal M_h\)-compatible data distribution \(\mathsf D\) and every finite group family satisfying \(|\mathcal G|\le\lceil\eps^{-\kappa}\rceil\), there is a randomized learner using at most \(C\eps^{-5}\) samples whose output \(\Pi_S\) satisfies
\[
\Prob\left[
\MCErr^\Gamma_{\mathsf D}(\Pi_S;\mathcal G)\le\eps
\right]
\ge
\frac23.
\]
The exact oracle in Proposition~\ref{prop:upper-mvs} gives training time polynomial in \(1/\eps\).
\end{corollary}

\begin{proof}
Apply Corollary~\ref{cor:upper-sample} and Theorem~\ref{thm:upper-polytime} using Proposition~\ref{prop:upper-mvs}.
\end{proof}

\subsection{Quantile, CVaR}\label{sec:quantile-cvar}

Fix \(\tau\in(0,1)\) and constants \(0<h<1<H<\infty\).  Let \(\mathcal M_{h,H}\) be the class of probability laws on \([0,1]\) with densities \(f\) satisfying
\[
h\le f(y)\le H
\qquad\text{for almost every }y\in[0,1].
\]
We use the prediction space
\[
\mathcal P=[0,1]^2.
\]
For \(\mu\in\mathcal M_{h,H}\), define the \(\tau\)-quantile
\[
q_\tau(\mu)
:=
\inf\{q:\mu(Y\le q)\ge\tau\}.
\]
Use the upper-tail convention
\[
\CVaR_\tau(\mu)
:=
q_\tau(\mu)
+
\frac{\E_{\mu}(Y-q_\tau(\mu))_+}{1-\tau}.
\]
Consider the scoring function
\[
L_\tau(q,y)
:=
q+\frac{(y-q)_+}{1-\tau}.
\]
Over \(\mathcal M_{h,H}\), \(L_\tau\) is strictly consistent for \(q_\tau\), and its Bayes risk is \(\CVaR_\tau\)~\citep{noarov2023scope}.  Thus
\[
(q_\tau,\CVaR_\tau)
\]
is a Bayes pair.  This is the \(k=2\) case of the sequential framework with residual functions
\[
R_1(\varnothing,q,y)=\1\{y\le q\}-\tau,
\qquad
R_2(q,r,y)=r-L_\tau(q,y).
\]

We now construct a two-dimensional local witness.  For \(\lambda=(\lambda_1,\lambda_2)\) in a small neighborhood of \(0\in\R^2\), define the density
\[
f_\lambda(y)
=
\exp(\lambda_1y+\lambda_2y^2-A(\lambda)),
\qquad
0\le y\le1,
\]
where \(A(\lambda)\) is the log normalizer.  Let \(\mu_\lambda\) be the corresponding distribution.
Since \(f_0\equiv1\), we may choose the neighborhood of \(0\) small enough that \(\mu_\lambda\in\mathcal M_{h,H}\) throughout it.

Define
\[
\Psi(\lambda)
:=
\left(
q_\tau(\mu_\lambda),
\CVaR_\tau(\mu_\lambda)
\right).
\]
At \(\lambda=0\), \(\mu_0=\Unif[0,1]\), so
\[
\Psi(0)
=
\left(
\tau,\frac{1+\tau}{2}
\right).
\]

\begin{lemma}[Local Two-Dimensional Parameterization]\label{lem:qc-diffeo}
The Jacobian \(D\Psi(0)\) is nonsingular.  Consequently, there exists a nondegenerate rectangle
\[
\mathcal R_{\mathrm{QC}}
\subset
\operatorname{range}(\Psi)
\]
around \(\bigl(\tau,(1+\tau)/2\bigr)\) and a smooth inverse map
\[
\lambda=\lambda(v),
\qquad
v=(v_1,v_2)\in\mathcal R_{\mathrm{QC}},
\]
such that
\[
q_\tau(\mu_{\lambda(v)})=v_1,
\qquad
\CVaR_\tau(\mu_{\lambda(v)})=v_2.
\]
\end{lemma}

\begin{proof}
Since \(f_\lambda(y)\) is smooth in \((\lambda,y)\) and strictly positive, the implicit equation \(\int_0^{q_\tau(\mu_\lambda)}f_\lambda(y)\,dy=\tau\) and the implicit function theorem show that \(\lambda\mapsto q_\tau(\mu_\lambda)\) is smooth near \(0\).  The defining integral for \(\CVaR_\tau(\mu_\lambda)\) is then smooth as well, so \(\Psi\) is smooth near \(0\).

At \(\lambda=0\), direct differentiation gives
\[
D\Psi(0)
=
\begin{pmatrix}
\dfrac{\tau(1-\tau)}2
&
\dfrac{\tau(1-\tau^2)}3
\\[1.1em]
\dfrac{(1-\tau)(2\tau+1)}{12}
&
\dfrac{(1-\tau)(\tau+1)^2}{12}
\end{pmatrix}.
\]
Its determinant is
\[
\det D\Psi(0)
=
\frac{\tau(1-\tau)^3(1+\tau)}{72}
>
0.
\]
The inverse function theorem gives the claim.
\end{proof}

For \(v\in\mathcal R_{\mathrm{QC}}\), define
\[
\mu_v:=\mu_{\lambda(v)}.
\]

It remains to check that the local parameterization from Lemma~\ref{lem:qc-diffeo} has the residual and KL properties required by Assumption~\ref{ass:witness}.

\begin{proposition}[Quantile, CVaR: Witness]\label{prop:qc-witness}
For \(\mathcal R_{\mathrm{QC}}\) sufficiently small, the family
\[
\{\mu_v:v\in\mathcal R_{\mathrm{QC}}\}
\]
is contained in \(\mathcal M_{h,H}\) and satisfies Assumption~\ref{ass:witness} for the property
\[
(q_\tau,\CVaR_\tau).
\]
\end{proposition}

\begin{proof}
By Lemma~\ref{lem:qc-diffeo}, \(\Gamma(\mu_v)=v\), as required by Assumption~\ref{ass:witness}(i).

By the choice of the parameter neighborhood, \(\mu_v\in\mathcal M_{h,H}\) for every \(v\in\mathcal R_{\mathrm{QC}}\).
Let \(F_v\) be the CDF of \(\mu_v\).  Since \(F_v(v_1)=\tau\),
\[
R_1(v_{<1},q,\mu_v)
=
F_v(q)-\tau
=
F_v(q)-F_v(v_1),
\qquad
R_2(v_1,r,\mu_v)=r-v_2.
\]
Therefore,
\[
\sgn(q-v_1)\left(R_1(v_{<1},q,\mu_v)-R_1(v_{<1},v_1,\mu_v)\right)
\ge
h|q-v_1|.
\]
The second residual evaluated at the true prefix satisfies
\[
\sgn(r-v_2)\left(R_2(v_1,r,\mu_v)-R_2(v_1,v_2,\mu_v)\right)
=
|r-v_2|.
\]
Thus Assumption~\ref{ass:witness}(ii) holds.  The bound in condition~(iii) follows from
\[
\left|R_1(v_{<1},q,\mu_v)-R_1(v_{<1},v_1,\mu_v)\right|
=
|F_v(q)-F_v(v_1)|
\le
H|q-v_1|
\]
and
\[
\left|R_2(v_1,r,\mu_v)-R_2(v_1,v_2,\mu_v)\right|
=
|r-v_2|.
\]

It remains to verify Assumption~\ref{ass:witness}(iv).  The first coordinate has no preceding predictions.  For the second coordinate, fix \(y\in[0,1]\).  The map
\[
q\mapsto L_\tau(q,y)
=
q+\frac{(y-q)_+}{1-\tau}
\]
is Lipschitz with constant
\[
L_\tau^*
:=
\max\left\{1,\frac{\tau}{1-\tau}\right\}.
\]
Hence
\[
\left|
\E_{\mu_v}L_\tau(p_1,Y)
-
\E_{\mu_v}L_\tau(v_1,Y)
\right|
\le
L_\tau^*|p_1-v_1|.
\]
Since
\[
R_2(p_1,p_2,\mu_v)-R_2(v_1,p_2,\mu_v)
=
-
\left(
\E_{\mu_v}L_\tau(p_1,Y)
-
\E_{\mu_v}L_\tau(v_1,Y)
\right),
\]
this verifies Assumption~\ref{ass:witness}(iv).

Finally, the family \(\{f_\lambda\}\) is a regular two-parameter exponential family with bounded sufficient statistics \(y\) and \(y^2\).  On a compact neighborhood of \(0\),
\[
\KL(\mu_\lambda\,\|\,\mu_{\lambda'})
\le
C\|\lambda-\lambda'\|_2^2.
\]
Because \(v\mapsto\lambda(v)\) is smooth on \(\mathcal R_{\mathrm{QC}}\),
\[
\KL(\mu_v\,\|\,\mu_{v'})
\le
C'\|v-v'\|_2^2.
\]
Thus Assumption~\ref{ass:witness}(v) holds.
\end{proof}

\begin{corollary}[Quantile, CVaR: Lower Bound]
Fix \(\tau\in(0,1)\), constants \(0<h<1<H<\infty\), and \(\kappa>0\).  There exist constants \(c,C,\eps_0>0\), depending only on \(\tau,h,H\), and \(\kappa\), such that the following holds for every \(0<\eps\le\eps_0\).  Let
\[
\Gamma=(q_\tau,\CVaR_\tau)
\]
be the property over \(\mathcal M_{h,H}\).  One can construct a finite context space, a binary group family of size at most \(\eps^{-\kappa}\), and a finite collection of \(\mathcal M_{h,H}\)-compatible data distributions, such that every learner achieving \(\Gamma\)-ECE at most \(\eps\) with probability at least \(2/3\) must use
\[
n
\ge
c\frac{\eps^{-4}}{\log^4(C/\eps)}.
\]
Equivalently,
\[
\SC^{(\kappa)}_{\Gamma,\mathcal M_{h,H}}(\eps)
=
\widetilde{\Omega}(\eps^{-4}).
\]
\end{corollary}

\begin{proof}
Apply Theorem~\ref{thm:lower-bound} using Proposition~\ref{prop:qc-witness}.
\end{proof}

For the upper bound, we verify the required conditions over the same class \(\mathcal M_{h,H}\).

\begin{proposition}[Quantile, CVaR: Conditions for the Upper Bound]\label{prop:upper-qc}
The regularity conditions in Assumption~\ref{ass:upper-analytic} hold over \(\mathcal M_{h,H}\).  Moreover, the linear optimization oracle in Assumption~\ref{ass:upper-oracle} has a polynomial-time implementation to any additive accuracy \(\xi>0\).
\end{proposition}

\begin{proof}
Recall that
\[
R((q,r),y)=\left(\1\{y\le q\}-\tau,\ r-L_\tau(q,y)\right).
\]
The class \(\mathcal M_{h,H}\) is convex and compact in the weak topology.  For each fixed \((q,r)\), the map
\[
\mu\mapsto F_\mu(q)-\tau
\]
is affine and continuous on \(\mathcal M_{h,H}\), because every law in this class has no atom at \(q\).  The map
\[
\mu\mapsto \E_\mu[r-L_\tau(q,Y)]
\]
is affine and continuous because \(L_\tau(q,\cdot)\) is bounded and continuous.  These observations verify Assumption~\ref{ass:upper-analytic}(i).  Condition~(ii) also holds because \(q,r,y\in[0,1]\) and \(\tau\in(0,1)\).

To verify Assumption~\ref{ass:upper-analytic}(iii), let \(F_\mu\) be the CDF of \(\mu\).  For \(q,q'\in[0,1]\),
\[
\left|R_1(\varnothing,q,\mu)-R_1(\varnothing,q',\mu)\right|
=
|F_\mu(q)-F_\mu(q')|
\le H|q-q'|.
\]
Also, the map \(q\mapsto L_\tau(q,y)\) is Lipschitz uniformly in \(y\), with a constant depending only on \(\tau\).
Therefore \(R(\cdot,\mu)\) is uniformly Lipschitz in \((q,r)\).  A rectangular grid with mesh \(O(1/Q)\) satisfies the required approximation bound and has \(O(Q^2)\) points.

It remains to implement the oracle.  Given coefficients \(c_{p,j}\), write each grid point as \(p=(q_p,r_p)\).  The integrand
\[
\phi(y)=\sum_{p\in\mathcal P_Q}
\left[
 c_{p,1}(\1\{y\le q_p\}-\tau)
 +c_{p,2}\bigl(r_p-q_p-(y-q_p)_+/(1-\tau)\bigr)
\right]
\]
is affine on each open interval between consecutive distinct grid values \(q_p\), and it may have jump discontinuities at those values.  Since there are only finitely many such points, their values do not affect the oracle objective.  The oracle must maximize
\[
\int_0^1 \phi(y)f(y)\,dy
\qquad\text{subject to}\qquad
h\le f\le H,
\quad \int_0^1 f=1.
\]
Write \(f=h+(H-h)u\), where \(0\le u\le1\) and
\[
\int_0^1u(y)\,dy=\frac{1-h}{H-h}=:m_0.
\]
The objective differs by the constant \(h\int\phi\) from maximizing \(\int\phi u\) over \(0\le u\le1\) with mass \(m_0\).  For a threshold \(t\), let
\[
A(t):=\int_0^1\1\{\phi(y)>t\}\,dy,
\qquad
B(t):=\int_0^1\1\{\phi(y)\ge t\}\,dy.
\]
Choose \(t\) such that \(A(t)\le m_0\le B(t)\), and set
\[
\gamma
:=
\begin{cases}
\dfrac{m_0-A(t)}{B(t)-A(t)},&B(t)>A(t),\\[0.8em]
0,&B(t)=A(t).
\end{cases}
\]
Then
\[
u_t(y)
:=
\1\{\phi(y)>t\}
+
\gamma\1\{\phi(y)=t\}
\]
has mass \(m_0\).  It is optimal by the bathtub principle, which places the available mass where \(\phi\) is largest.  This includes the case in which \(\phi\) is constant on an interval: \(\gamma\) supplies exactly the required fraction of that level set.

Because \(\phi\) has \(O(|\mathcal P_Q|)\) affine pieces, \(A(t)\) is affine on each interval between consecutive critical values.  These critical values are the values of \(\phi\) at the endpoints of its affine pieces, together with the values on any constant pieces.  Sorting them and solving one linear equation locates an admissible \(t\) and \(\gamma\) in polynomial time.  The resulting density \(f_t=h+(H-h)u_t\) is piecewise constant on \(O(|\mathcal P_Q|)\) intervals, which gives a polynomial-size representation.  Integrating the residuals interval by interval computes the objective and all residual expectations to additive accuracy \(\xi\) in time polynomial in \(|\mathcal P_Q|\) and \(\log(1/\xi)\).  This implements Assumption~\ref{ass:upper-oracle}.
\end{proof}

\begin{corollary}[Quantile, CVaR: Upper Bound]\label{cor:upper-qc}
Fix \(\tau\in(0,1)\), constants \(0<h<1<H<\infty\), and \(\kappa>0\).  There exist constants \(C,\eps_0>0\), depending only on \(\tau,h,H\), and \(\kappa\), such that the following holds for every \(0<\eps\le\eps_0\).  Let
\[
\Gamma=(q_\tau,\CVaR_\tau)
\]
be the property over \(\mathcal M_{h,H}\).  For every \(\mathcal M_{h,H}\)-compatible data distribution \(\mathsf D\) and every finite group family satisfying \(|\mathcal G|\le\lceil\eps^{-\kappa}\rceil\), there is a randomized learner using at most \(C\eps^{-4}\) samples whose output \(\Pi_S\) satisfies
\[
\Prob\left[
\MCErr^\Gamma_{\mathsf D}(\Pi_S;\mathcal G)\le\eps
\right]
\ge
\frac23.
\]
The oracle in Proposition~\ref{prop:upper-qc} gives training time polynomial in \(1/\eps\).
\end{corollary}

\begin{proof}
Apply Corollary~\ref{cor:upper-sample} and Theorem~\ref{thm:upper-polytime} using Proposition~\ref{prop:upper-qc}.
\end{proof}

\subsection*{Statement of AI Use}

GPT-5.6 Sol Pro was used during manuscript preparation for brainstorming, drafting, editing, and formatting assistance, including preliminary drafts of proof arguments. The LLM-generated material was not used without author review. All proof arguments were independently checked, corrected, and finalized by the authors, who take full responsibility for the correctness, originality, and presentation of the paper.

\bibliographystyle{plainnat}
\bibliography{refs-without-doi-url-isbn}

\appendix
\section{Comparison with Multi-Level Stochastic Optimization}\label{sec:optcomparison}
Here, we discuss the difference between $k$-level multicalibration problem and the well-studied $k$-level stochastic composite optimization problem. The key distinction is that the levels are \emph{sequential} in
compositional optimization but \emph{Cartesian} in multicalibration.

A \(k\)-level composition optimization problem involves solving
\[
    F(x)=f_1\circ f_2\circ\cdots\circ f_k(x),
\]
given access to stochastic evaluations of the gradients and function values of functions $f_j$. In this setup, typically, the quantity of interest is a fixed point \(x\) that determines only one chain of intermediate values,
\[
    y_k=x,
    \qquad
    y_{j-1}=f_j(y_j),
    \qquad
    F(x)=f_1(y_1).
\]
The optimization algorithm must estimate and track the quantities along
this chain, but its statistical goal is still to find a single point
\(\widehat{x}_N\) (where $N$ is the sample or iteration complexity) satisfying a first-order condition, such as
\[
    \mathbb{E}\bigl[\|\nabla F(\widehat{x}_N)\|\bigr]
    \leq \varepsilon.
\]
It is not required to certify stationarity separately for every possible
vector of intermediate values \((y_1,\ldots,y_{k-1})\).  Smoothness allows
the estimation errors arising at the different levels to be propagated
along the chain and controlled on the same stochastic time scale.  For
the stationarity criterion considered in the multi-level optimization
literature, the resulting bound has the schematic form
\[
    \mathbb{E}\bigl[\|\nabla F(\widehat{x}_N)\|^2\bigr]
    \lesssim N^{-1/2},
    \qquad\text{and hence}\qquad
    \mathbb{E}\bigl[\|\nabla F(\widehat{x}_N)\|\bigr]
    \lesssim N^{-1/4}.
\]
Thus \(N\gtrsim\varepsilon^{-4}\) suffices, independently of the number
of compositional levels in the exponent of \(\varepsilon\).  Additional
levels make the gradient estimator more involved, but they do not create
additional \(\varepsilon\)-scale locations at which the desired guarantee
must hold.  This sequential propagation of estimation error also underlies
level-independent accuracy rates in statistical and sample-average
analyses of compound objectives
\citep{dentcheva2017statistical,ermoliev2013sample}, as well as in
stochastic algorithms for nested composition problems
\citep{ruszczynski2021stochastic,chen2021solving,balasubramanian2022stochastic}.

Multicalibration has a different geometry.  The predictor outputs a
prediction vector
\[
    p=(p_1,\ldots,p_k)\in\mathcal P_1\times\cdots\times\mathcal P_k,
\]
and calibration is imposed after conditioning on the \emph{entire} prediction
vector.  Although the residuals are sequentially conditionally
identifiable, the collection of prediction values over which calibration must
be controlled is a Cartesian product.  Discretizing each coordinate at
resolution \(Q^{-1}\) therefore produces
\[
    |\mathcal P_Q|\asymp Q^k
\]
prediction values.  Correspondingly, the statistical error has the form
\[
    \operatorname{MCErr}^{\Gamma}_{D}(\Pi;\mathcal G)
    \lesssim
    \frac{1}{Q}
    +
    \sqrt{\frac{Q^k+\log|\mathcal G|}{n}}.
\]
The difference from ordinary mean estimation can be seen directly from
the second term.  Suppose, for intuition, that the \(M=Q^k\) prediction
buckets have comparable probability.  A bucket then contains about
\(n/M\) observations.  Its conditional residual is estimated with error
of order \(\sqrt{M/n}\); after multiplying by the bucket probability
\(1/M\), its contribution to the calibration error is of order
\(1/\sqrt{Mn}\).  Since the calibration error sums absolute residual
contributions over all \(M\) buckets, their total stochastic contribution
is
\[
    M\cdot\frac{1}{\sqrt{Mn}}
    =
    \sqrt{\frac{M}{n}}
    =
    \sqrt{\frac{Q^k}{n}}.
\]
To make the discretization error at most \(\varepsilon\), one takes
\(Q\asymp\varepsilon^{-1}\).  Requiring the stochastic term to be at most
\(\varepsilon\) then gives
\[
    \sqrt{\frac{Q^k}{n}}\lesssim\varepsilon
    \qquad\Longrightarrow\qquad
    n\gtrsim\frac{Q^k}{\varepsilon^2}
    \asymp
    \varepsilon^{-(k+2)}.
\]
Thus the factor \(\varepsilon^{-2}\) is the usual cost of estimating
expectations to accuracy \(\varepsilon\), while the additional factor
\(\varepsilon^{-k}\) is the number of distinguishable prediction values at
that resolution.  Each additional calibrated level adds another prediction
coordinate, multiplies the size of the joint grid by
\(\varepsilon^{-1}\), and therefore increases the exponent by one.

The lower bound confirms that this difference is intrinsic rather than an
artifact of discretization.  At resolution \(Q^{-1}\), one can encode
independent perturbations of size \(\Theta(Q^{-1})\) across
\(\Theta(Q^k)\) grid points.  The logarithm of the number of
distinguishable alternatives is then \(\Omega(Q^k)\), whereas one sample
contributes only \(O(Q^{-2})\) Kullback--Leibler information.  Fano's
inequality consequently requires
\[
    nQ^{-2}\gtrsim Q^k,
    \qquad\text{or equivalently}\qquad
    n\gtrsim Q^{k+2}.
\]
The reduction from calibration error to prediction accuracy loses a factor
of \(\log Q\), so a construction at resolution \(Q^{-1}\) applies when
\(\varepsilon\asymp 1/(Q\log Q)\).  Thus
\[
    Q\asymp \frac{1}{\varepsilon\log(1/\varepsilon)},
    \qquad
    n\gtrsim
    \frac{\varepsilon^{-(k+2)}}{\log^{k+2}(1/\varepsilon)}
    =\widetilde{\Omega}\!\left(\varepsilon^{-(k+2)}\right).
\]

In summary, stochastic
composition optimization follows one nested chain and asks for
stationarity at one output point, whereas multicalibration must resolve
and control a \(k\)-dimensional Cartesian grid of prediction values.

\section{Measurability of the Online Forecaster}\label{app:upper-measurability}

We justify that the prediction rules in Algorithm~\ref{alg:upper-online} can
be chosen jointly measurably.  For the fixed grid \(\mathcal P_Q\), define
the finite-dimensional residual image
\[
\mathcal K_Q
:=
\left\{
\bigl(R_j(p_{<j},p_j,\mu)\bigr)_{(p,j)\in\mathcal P_Q\times[k]}
:
\mu\in\mathcal M
\right\}
\subseteq
\R^{k|\mathcal P_Q|}.
\]
Assumption~\ref{ass:upper-analytic}(i) implies that \(\mathcal K_Q\) is
compact, because it is the continuous image of the compact set
\(\mathcal M\).  For \(w\in\Delta(\mathcal G\times\Sigma)\),
\(z=(z_g)_{g\in\mathcal G}\in[0,1]^{\mathcal G}\),
\(\pi\in\Delta(\mathcal P_Q)\), and
\(r=(r_{p,j})_{p,j}\in\mathcal K_Q\), set
\[
F(\pi,r;w,z)
:=
\sum_{p\in\mathcal P_Q}\pi_p
\sum_{g\in\mathcal G}\sum_{s\in\Sigma}
w(g,s)z_g
\sum_{j=1}^k s_{p,j}r_{p,j}.
\]
This function is jointly continuous on finite-dimensional compact spaces.
Consequently,
\[
\Phi(\pi;w,z)
:=
\max_{r\in\mathcal K_Q}F(\pi,r;w,z),
\qquad
V(w,z)
:=
\min_{\pi\in\Delta(\mathcal P_Q)}\Phi(\pi;w,z),
\]
are continuous by the maximum theorem.  For every \(\rho\ge0\), the
correspondence
\[
\mathcal A_\rho(w,z)
:=
\left\{
\pi\in\Delta(\mathcal P_Q):
\Phi(\pi;w,z)\le V(w,z)+\rho
\right\}
\]
has nonempty compact values and a measurable graph.  The measurable
selection theorem therefore provides a Borel map
\(a_\rho(w,z)\in\mathcal A_\rho(w,z)\).

At round \(t\), take
\[
z(x):=(g(x))_{g\in\mathcal G},
\qquad
\pi_t(\cdot\mid x):=a_\rho(w_t,z(x)).
\]
The map \(x\mapsto z(x)\) is measurable, and \(w_t\) is a measurable
function of the cumulative gains through round \(t-1\).  Hence
\((\text{history},x)\mapsto\pi_t(\cdot\mid x)\) is jointly measurable and
satisfies \eqref{eq:upper-minimax}.
\end{document}